\documentclass{article}
\usepackage[margin=3cm]{geometry}
\usepackage{natbib}
\usepackage{subcaption}
\usepackage{hyperref}

\usepackage{enumitem}
\setlist{nosep} % to reduce space in lists

\usepackage{amsmath,bbm,bm}
\usepackage{amsfonts}
\usepackage{amsthm}
\usepackage{mathtools}
\usepackage{algorithm}
\usepackage{algpascal}
\usepackage{algorithmicx}
\usepackage{algpseudocode}
\usepackage{tabularx}
\algdef{SE}[SUBALG]{Indent}{EndIndent}{}{\algorithmicend\ }%
\algtext*{Indent}
\algtext*{EndIndent}

\newtheorem{thm}{Theorem}%[section]
\newtheorem{lem}{Lemma}
\newtheorem{prop}{Proposition}

\newtheorem{aspt}{Assumption}

\newtheorem{defn}{Definition}

\usepackage{xcolor}
\newcount\comments  % 0 suppresses notes to selves in text
\newcommand{\genComment}[2]{\ifnum\comments=1{\textcolor{#1}{\textsf{\footnotesize #2}}}\fi}

\def\eqref#1{equation~\ref{#1}}
\def\1{\bm{1}}

\DeclareMathAlphabet{\mathsfit}{\encodingdefault}{\sfdefault}{m}{sl}
\SetMathAlphabet{\mathsfit}{bold}{\encodingdefault}{\sfdefault}{bx}{n}

\def\gA{{\mathcal{A}}}

\def\gC{{\mathcal{C}}}
\def\gD{{\mathcal{D}}}

\def\gF{{\mathcal{F}}}

\def\gL{{\mathcal{L}}}
\def\gM{{\mathcal{M}}}

\def\gS{{\mathcal{S}}}

\newcommand{\E}{\mathbb{E}}

\DeclareMathOperator*{\argmax}{arg\,max}
\DeclareMathOperator*{\argmin}{arg\,min}

\newcommand{\expl}{\operatorname{expl}}
\newcommand{\barbelow}[1]{\text{\b{$#1$}}}
\newcommand\numberthis{\addtocounter{equation}{1}\tag{\theequation}}

\title{Sample Efficient Myopic Exploration Through Multitask Reinforcement Learning with Diverse Tasks}

\author{\textbf{Ziping Xu}\footnote{This work was done when Z.X. was a Ph.D. student at the University of Michigan.}$^{*, 1}$, \textbf{Zifan Xu}$^{2}$, \textbf{Runxuan Jiang}$^{3}$, \textbf{Peter Stone}$^{2, 4}$ \& \textbf{Ambuj Tewari}$^{3}$ \\
$^1$Harvard University, $^2$University of Texas at Austin, $^3$University of Michigan, $^4$Sony AI
}
\date{}

\begin{document}

\maketitle
\begin{abstract}
    Multitask Reinforcement Learning (MTRL) approaches have gained increasing attention for its wide applications in many important Reinforcement Learning (RL) tasks. However, while recent advancements in MTRL theory have focused on the improved statistical efficiency by assuming a shared structure across tasks, exploration--a crucial aspect of RL--has been largely overlooked. This paper addresses this gap by showing that when an agent is trained on a sufficiently {\it diverse} set of tasks, a generic  policy-sharing algorithm with myopic exploration design like $\epsilon$-greedy that are inefficient in general can be sample-efficient for MTRL. To the best of our knowledge, this is the first theoretical demonstration of the "exploration benefits" of MTRL. It may also shed light on the enigmatic success of the wide applications of myopic exploration in practice. To validate the role of diversity, we conduct experiments on synthetic robotic control environments, where the diverse task set aligns with the task selection by automatic curriculum learning, which is empirically shown to improve sample-efficiency.
\end{abstract}
\section{Introduction}

Reinforcement Learning often involves solving multitask problems. For instance, robotic control agents are trained to simultaneously solve multiple goals in multi-goal environments \citep{andreas2017modular,andrychowicz2017hindsight}. In mobile health applications, RL is employed to personalize sequences of treatments, treating each patient as a distinct task \citep{yom2017encouraging,forman2019can,liao2020personalized,ghosh2023did}. Many algorithms \citep{andreas2017modular,andrychowicz2017hindsight,hessel2019multi,yang2020multi} have been designed to jointly learn from multiple tasks. These show significant improvement over those that learn each task individually. To provide explanations for such improvement, recent advancements in Multitask Reinforcement Learning (MTRL) theory study the improved statistical efficiency in estimating unknown parameters by assuming a shared structure across tasks \citep{brunskill2013sample,calandriello2014sparse,uehara2021representation,xu2021decision,zhang2021provably,lu2021power,agarwal2022provable,cheng2022provable,yang2022nearly}.
Similar setups originate from Multitask Supervised Learning, where it has been shown that learning from multiple tasks reduces the generalization error by a factor of $1/\sqrt{N}$ compared to single-task learning with $N$ being the total number of tasks \citep{maurer2016benefit,du2020few}. Nevertheless, these studies overlook an essential aspect of RL, namely exploration.

% In this paper, we address that gap by demonstrating that the complexity of exploration design can be significantly reduced if an algorithm learns 
 To understand how learning from multiple tasks, as opposed to single-task learning, could potentially benefit exploration design, we consider a generic MTRL scenario, where an algorithm interacts with a task set $\mathcal{M}$ in rounds $T$. %Tasks in $\mathcal{M}$ share only the state and action spaces, making it impossible to exploit any shared structure in general. 
{\it In each round, the algorithm chooses an exploratory policy $\pi$ that is used to collect one episode worth of data in a task $M \in \mathcal{M}$ of its own choice. A sample-efficient algorithm should output a near-optimal policy for each task in a polynomial number of rounds.}
% We assume different tasks share only the state and action space, making it impossible to exploit the structural similarities.

Exploration design plays an important role in achieving sample-efficient learning. Previous sample-efficient algorithms for single-task learning ($|\mathcal{M}| = 1$) heavily rely on strategic design on the exploratory policies, such as Optimism in Face of Uncertainty (OFU) \citep{auer2008near,bartlett2009regal,dann2017unifying} and Posterior Sampling \citep{russo2014learning,osband2017posterior}. Strategic design is criticized for either being restricted to environments with strong structural assumptions, or involving intractable computation oracle, such as non-convex optimization \citep{jiang2018pac,jin2021bellman}. For instance, GOLF \citep{jin2021bellman}, a model-free general function approximation online learning algorithm, involves finding the optimal Q-value function $f$ in a function class $\gF$. It requires the following two intractable computation oracles in (\ref{equ:GOLF}).
\begin{equation}
    \mathcal{F}^t=\left\{f \in \mathcal{F}: f \text{ has low empirical Bellman error}\right\} \text{ and } f^{(t)} = \argmax_{f \in \gF^t} f(s_1, \pi(s_1 \mid f)). \label{equ:GOLF}
\end{equation}

\vspace{-3mm}

In contrast, myopic exploration design like $\epsilon$-greedy that injects random noise to a current greedy policy is easy to implement and performs well in a wide range of applications \citep{mnih2015human, kalashnikov2018scalable}, while it is shown to have exponential sample complexity in the worst case for single-task learning \citep{osband2019deep}. Throughout the paper, we ask the main question:
\begin{center}
    {\it Can algorithms with myopic exploration design be sample-efficient for MTRL?}
\end{center}
\iffalse
\begin{multicols}{2}
  \begin{align*}
    &\text{GOLF \citep{jin2021bellman}}\\
    &f^{(t)} = \argmax_{f \in \gF^{(t-1)}}f(s_1, \pi_{f, 1}(s_1)) \numberthis \\
    & \pi^{(t)} = \pi^{f^{(t)}}
  \end{align*}\break
  \begin{align*}
    &\text{Myopic exploration}\\
    &f^{(t)} = \argmin_{f \in \gF} \gL_{\gD^{(t-1)}}(f) \numberthis \\
    & \pi^{(t)} = (1-\epsilon) \pi^{f^{(t)}} + \epsilon \text{Unif}
  \end{align*}
\end{multicols}
\fi

In this paper, we address the question by showing that {\it a simple algorithm that explores one task with $\epsilon$-greedy policies from other tasks can be sample-efficient if the task set $\mathcal{M}$ is adequately diverse.} Our results may shed some light on the longstanding mystery that $\epsilon$-greedy is successful in practice, while being shown sample-inefficient in theory. We argue that in {an} MTRL setting, $\epsilon$-greedy policies may no longer behave myopically, as they explore myopically around the optimal policies from other tasks. When the task set is adequately diverse, this exploration may provide sufficient coverage. It is worth-noting that this may partially explain the success of curriculum learning in RL \citep{narvekar2020curriculum}, that is the optimal policy of one task may easily produce a good exploration for the next task, thus reducing the complexity of exploration. This connection will be formally discussed later.
%Unlike algorithms in the previous literature that share parameters or feature extractors across tasks, our algorithm can be seen as sharing policies to encourage exploration. 

To summarize our contributions, {\it this work provides a general framework that converts the task of strategic exploration design into the task of constructing diverse task set in {an} MTRL setting}, where myopic exploration like $\epsilon$-greedy can also be sample-efficient. We discuss a sufficient diversity condition under a general value function approximation setting \citep{jin2021bellman,dann2022guarantees}. We show that the condition guarantees polynomial sample-complexity bound by running the aforementioned algorithm with myopic exploration design. We further discuss how to satisfy the diversity condition in different {case} studies, including tabular cases, linear cases and linear quadratic regulator cases. In the end, we validate our theory with experiments on synthetic robotic control environments, where we {observe} that a diverse task set is similar to the task selection of the state-of-the-art automatic curriculum learning algorithm, which has been empirically shown to improve sample efficiency. 
%This leads to significant reduction in the complexity of exploration design, as it is well-known that sample-efficient RL on a single task requires strategic exploration. In the previous literature, popular exploration design include  and Perturbation-based Exploration \citep{xiong2021randomized,xu2019worst}. However, these exploration strategies either heavily rely on strong structural assumptions on the underlying environments, which may lead to poor generalization performance, or involve intractable computation oracle, such as solving a non-convex optimization problem \citep{jiang2018pac,jin2021bellman}. On the other hand, myopic exploration approaches like $\epsilon$-greedy are easy to implement and perform well in a wide range of problems with practical interest \citep{mnih2015human, kalashnikov2018scalable}. In fact, our theoretical result may shed light on the longstanding mystery of why myopic exploration has shown significant success in empirical studies despite having exponential sample complexity in the worst case \citep{osband2019deep}. 
\section{Problem Setup}

The following are notations that will be used throughout the paper.

{\bf Notation.} For a positive integer $H$, we denote $[H] \coloneqq \{1, \dots, H\}$. For a discrete set $\mathcal{A}$, we denote $\Delta_{\mathcal{A}}$ by the set of distributions over $\mathcal{A}$. We use $\mathcal{O}$ and ${\Omega}$ to denote the asymptotic upper and lower bound notations and use $\tilde{\mathcal{O}}$ and $\tilde{{\Omega}}$ to hide the logarithmic dependence. Let $\{\1_i\}_{i \in [d]}$ be the standard basis that spans $\mathbb{R}^d$. We let $N_{\mathcal{F}}(\rho)$ denote the $\ell_{\infty}$ covering number of a function class $\mathcal{F}$ at scale $\rho$. For a class $\mathcal{F}$, we denote the $N$-times Cartesian product of $\mathcal{F}$ by $(\mathcal{F})^{\otimes N}$.

\subsection{Proposed Multitask Learning Scenario}
Throughout the paper, we consider each task as an episodic MDP denoted by $M = (\mathcal{S}, \mathcal{A}, H, P_M, R_M)$, where $\mathcal{S}$ is the state space, $\mathcal{A}$ is the action space, $H \in \mathbb{N}$ represents the horizon length in each episode, $P_M = (P_{h, M})_{h \in [H]}$ is the collection of transition kernels, and $R_M = (R_{h, M})_{h \in [H]}$ is the collection of immediate reward functions. Each $P_{h, M}: \gS \times \gA \mapsto \Delta(\gS)$ and each $R_{h, M}: \gS \times \gA \mapsto [0, 1]$. {Note that we consider a scenario in which all the tasks share the same state space, action space, and horizon length.} %When stateLet $S, A$ denote the cardinality of $|\mathcal{S}|, |\mathcal{A}|$, respectively.

An agent interacts with an MDP $M$ in the following way: starting with a fixed initial state $s_1$, at each step $h \in [H]$, the agent decides an action $a_h$ and the environment samples the next state $s_{h+1} \sim P_{h, M}(\cdot \mid s_h, a_h)$ and next reward $r_h = R_{h, M}(s_h, a_h)$. An episode is a sequence of states, actions, and rewards $(s_1, a_1, r_1, \dots, s_H, a_h, r_H, s_{H+1})$. In general, we assume that the sum of $r_h$ is upper bounded by 1 for any action sequence almost surely. The goal of an agent is to maximize the cumulative reward $\sum_{h = 1}^H r_h$ by optimizing their actions. 

The agent chooses actions based on \textit{Markovian policies} %\ambuj{is it ok to use "stationary" even though these policies depend on $h$?}
denoted by $\pi = (\pi_h)_{h \in [H]}$ and each $\pi_h$ is a mapping $\mathcal{S} \mapsto \Delta_{\mathcal{A}}$, where $\Delta_{\mathcal{A}}$ is the set of all distributions over $\mathcal{A}$. Let $\Pi$ denote the space of all such policies.
For a finite action space, we let $\pi_h(a \mid s)$ denote the probability of selecting action $a$ given state $s$ at the step $h$. In case of the infinite action space, we slightly abuse the notation by letting $\pi_h(\cdot \mid s)$ denote the density function.

{\bf Proposed {multitask} learning scenario and objective.} We consider the following multitask RL learning scenario. An algorithm interacts with a set of tasks $\mathcal{M}$ sequentially for $T$ rounds. At the each round $t$, the algorithm chooses an exploratory policy, which is used to collect data for one episode in a task $M \in \mathcal{M}$ of its own choice. At the end of $T$ rounds, the algorithm outputs a set of policies $\{\pi_M\}_{M \in \mathcal{M}}$. The goal of an algorithm is to learn a near-optimal policy $\pi_M$ for each task $M \in \mathcal{M}$. The sample complexity of an algorithm is defined as follows.

\begin{defn}[MTRL Sample Complexity]
An algorithm {$\gL$} is said to have sample-complexity of ${\mathcal{C}_{\gM}^{\gL}}:\mathbb{R} \times \mathbb{R} \mapsto \mathbb{N}$ for a task set $\mathcal{M}$ if for any $\beta > 0, \delta \in (0, 1)$, it outputs a $\beta$-optimal policy $\pi_M$ for each MDP $M \in \mathcal{M}$ with probability at least $1-\delta$, by interacting with the task set for ${\mathcal{C}_{\gM}^{\gL}}(\beta, \delta)$ rounds. {We omit the notations for $\gL$ and $\gM$ when they are clear from the context to ease presentation.}
% For any $\beta > 0$, a sample-efficient algorithm outputs a $\beta$-optimal policy $\pi_M$ for each $M \in \mathcal{M}$ with high probability by interacting with the task set for $T$ rounds with $T$ being polynomial in the parameter of interests. For the tabular case, where the state and action spaces are finite, $T$ should be polynomial in $|\mathcal{S}|$, $|\mathcal{A}|$, $H$, $|\mathcal{M}|$, and $1/\beta$ for a sample-efficient learning.
\end{defn}

A sample-efficient algorithm should have a sample-complexity polynomial in the parameters of interests. For the tabular case, where the state space and action space are finite, $\mathcal{C}(\beta, \delta)$ should be polynomial in $|\mathcal{S}|$, $|\mathcal{A}|$, $|\mathcal{M}|$, H, and $1/\beta$ for a sample-efficient learning. Current state-of-the-art algorithm \citep{zhang2021reinforcement} on a single-task tabular MDP achieves sample-complexity of $\tilde{\mathcal{O}}(|\mathcal{S}||\mathcal{A}|/\beta^2)$\footnote{This bound is under the regime with $1/\beta \gg |\mathcal{S}|$}. This bound translates to a MTRL sample-complexity bound of $\tilde{\mathcal{O}}(|\mathcal{M}||\mathcal{S}||\mathcal{A}|/\beta^2)$ by running their algorithm individually for each $M \in \mathcal{M}$. However, their exploration design closely follows the principle of Optimism in Face of Uncertainty, which is normally criticized for over-exploring. 

\subsection{Value Function Approximation}

We consider the setting where value functions are approximated by general function classes. Denote the value function of an MDP $M$ with respect to a policy $\pi$ by
$$
    \begin{aligned} Q_{h, M}^\pi(s, a) & =\mathbb{E}_\pi^M\left[r_h+V_{h+1, M}^\pi\left(s_{h+1}\right) \mid s_h=s, a_h=a\right] \\ 
    V_{h, M}^\pi(s) & =\mathbb{E}_\pi^M\left[Q_{h, M}^\pi\left(s_h, a_h\right) \mid s_h=s\right],\end{aligned}
$$
where by $\mathbb{E}_\pi^M$, we take expectation over the randomness of trajectories sampled by policy $\pi$ on MDP $M$ and we let $V^{\pi}_{H+1, M}(s) \equiv 0$ for all $s \in \mathcal{S}$ and $\pi \in \Pi$. We denote the optimal policy for MDP $M$ by $\pi^*_M$. The corresponding value functions are denoted by $V_{h, M}^*$ and $Q_{h, M}^*$, which is shown to satisfy Bellman Equation $\mathcal{T}^M_h Q_{h+1, M}^* = Q_{h, M}^*$, where
$
    \text{for any } g: \mathcal{S} \times \mathcal{A} \mapsto [0, 1],  \left(\mathcal{T}^M_h g\right)(s, a)=\mathbb{E}[r_h+\max _{a^{\prime} \in \mathcal{A}} g\left(s_{h+1}, a^{\prime}\right) \mid s_h=s, a_h=a].
$

The agent has access to a collection of function classes $\mathcal{F} = (\mathcal{F}_h: \mathcal{S} \times \mathcal{A} \mapsto [0, H])_{h \in [H+1]}$. We assume that different tasks share the same set of function class.  For each $f \in \mathcal{F}$, we denote {by} $f_h \in \mathcal{F}_h$ the $h$-th component of the function $f$. We let $\pi^f = \{\pi_h^f\}_{h \in[H]}$ be the greedy policy with $\pi_h^{f}(s) = \argmax_{a \in \mathcal{A}} f_h(s, a)$. When it is clear from the context, we slightly abuse the notation and let $f \in (\mathcal{F})^{\otimes|\mathcal{M}|}$ be a joint function for all the tasks. We further let $f_{M}$ denote the function for the task $M$ and $f_{h, M}$ {denote} its $h$-th component.

Define Bellman error operator $\mathcal{E}_{h}^M$ such that $\mathcal{E}_h^M f = f_h - \mathcal{T}_h^M f_{h+1}$ for any $f \in \mathcal{F}$.  The goal of the learning algorithm is to approximate $Q_{h, M}^*$ through the function class $\mathcal{F}_h$ by minimizing the empirical Bellman error for each step $h$ and task $M$.

To provide theoretical guarantee on this practice, we make the following realizability and completeness assumptions. The two assumptions and their variants are commonly used in the literature \citep{dann2017unifying,jin2021bellman}.
 %Note that the mixture policy may not be a stat

\begin{aspt}[Realizability and Completeness] 
\label{aspt:function_class}
For any MDP $M$ considered in this paper, we assume $\mathcal{F}$ is realizable and complete under the Bellman operator such that $Q_{h, M}^* \in \mathcal{F}_h$ for all $h \in [H]$ and for every $h \in [H]$, $f_{h+1} \in \mathcal{F}_{h+1}$ there is a $f_h \in \mathcal{F}_h$ such that $f_h = \mathcal{T}_h^{M}f_{h+1}$. %\ambuj{this assumption is being made for every $M$?}
\end{aspt}
% \ziping{Can we make approximate assumption?}

\subsection{Myopic Exploration Design}
As opposed to carefully designed exploration, myopic exploration injects random noise to the current greedy policy. For a given greedy policy $\pi$, we use $\expl(\pi)$ to denote the myopic exploration policy based on $\pi$. Depending on the action space, the function $\expl$ can take different forms. The most common choice for finite action spaces is $\epsilon$-greedy, which mixes the greedy policy with a random action:
$
    \expl(\pi_h)(a \mid s) = (1-\epsilon_h) \pi_h(a \mid s) + \epsilon_h/A. \footnote{\text{Note that we consider a more general setup, where the exploration probability $\epsilon$ can depend on $h$}.}
$ As it is our main study of exploration strategies, we let $\expl$ be $\epsilon$-greedy function if not explicitly specified.
For a continuous action space, we consider exploration with Gaussian noise:
$
    \expl(\pi_h)(a \mid s) = (1-\epsilon_h) \pi_h(a \mid s) + \epsilon_h \exp(- a^2/{2\sigma_h^2})/\sqrt{2\pi\sigma_h^2}.
$ Gaussian noise is useful for Linear Quadratic Regulator (LQR) setting (discussed in Appendix \ref{app:case_studies}.)

\section{Multitask RL Algorithm with Policy-Sharing}

In this section, we introduce a generic algorithm (Algorithm \ref{alg:generic}) for the proposed multitask RL scenario without any strategic exploration, whose theoretical properties will be studied throughout the paper. In a typical single-task learning, a myopically exploring agent samples trajectories by running its current greedy policy estimated from the historical data equipped with naive explorations like $\epsilon$-greedy. % See Algorithm \ref{alg:generic-single-task} for an abstract of this learning framework.

\iffalse
\begin{algorithm}[h]
\caption{Generic Algorithm for Single-task Reinforcement Learning}
\begin{algorithmic}[1]
\For{round $t = 1, 2, \dots, T$}
\State Learn the greedy policy $\hat{\pi}_{t}$ from previous trajectories
\State Sample a new trajectory through exploration policy $\expl(\hat{\pi}_t)$
\EndFor
\end{algorithmic}
\label{alg:generic-single-task}
\end{algorithm}
\fi

In light of the exploration benefits of MTRL, we study Algorithm \ref{alg:generic} as a counterpart of the single-task learning scenario in the MTRL setting. Algorithm \ref{alg:generic} maintains a dataset for each MDP separately and different tasks interact in the following way: in each round, Algorithm \ref{alg:generic} explores every MDP with an exploratory policy that is the mixture (defined in Definition \ref{defn:mixture_policy}) of greedy policies of all the MDPs in the task set (Line 8). One way to interpret Algorithm \ref{alg:generic} is that we share knowledge across tasks by policy sharing instead of parameter sharing or feature extractor sharing in the previous literature.
\begin{defn}[Mixture Policy]
\label{defn:mixture_policy}
For a set of policies $\{\pi_i\}_{i = 1}^N$, we denote $\operatorname{Mixture}(\{\pi_i\}_{i = 1}^N)$ by the mixture of $N$ policies, such that before the start of an episode, it samples $I \sim \operatorname{Unif}([N])$, then runs policy $\pi_I$ for the rest of the episode.
\end{defn}
\vspace{-6pt}
The greedy policy is obtained from an offline learning oracle $\mathcal{Q}$ (Line 4) that maps a dataset $\mathcal{D} = \{(s_i, a_i, r_i, s_i')\}_{i = 1}^N$ to a function $f \in \mathcal{F}$, such that $\mathcal{Q}(\mathcal{D})$ is an approximate solution to the following minimization problem
$
      \argmin_{f \in \mathcal{F}} \sum_{i = 1}^N \left(f_{h_i}(s_i, a_i) - r_i - \max_{a' \in \mathcal{A}} f_{h_i+1}(s_i', a')\right)^2
$. In practice, one can run fitted Q-iteration for an approximate solution.

\vspace{-10pt}
\paragraph{Connection to Hindsight Experience Replay (HER) and multi-goal RL.}
{We provide some justifications for the choice of the mixture policy. HER is a common practice  \citep{andrychowicz2017hindsight} in the multi-goal RL setting \citep{andrychowicz2017hindsight,chane2021goal,liu2022goal}, where the reward distribution is a function of goal-parameters. HER relabels the rewards in trajectories in the experience buffer such that they were as if sampled for a different task. Notably, this exploration strategy is akin to randomly selecting task and collecting a trajectory on using its own epsilon-greedy policy, followed by relabeling rewards to simulate a trajectory on, which is equivalent to Algorithm \ref{alg:generic}. \cite{yang2022towards,zhumabekovensembling} also designed an algorithm with ensemble policies, which is shown to improve generalizability. It is worth noting that the ensemble implementation of Algorithm \ref{alg:generic} and that of \cite{zhumabekovensembling} differs. In Algorithm 1, policies are mixed trajectory-wise, with one policy randomly selected for the entire episode. In contrast, \cite{zhumabekovensembling} mixes policies on a step-wise basis for a continuous action space, selecting actions as a weighted average of actions chosen by different policies. This distinction is vital in showing the rigorous theoretical guarantee outlined in this paper.
}

\vspace{-10pt}
\paragraph{Connection to curriculum learning.} 
Curriculum learning is the approach that learns tasks in a specific order to improve the multitask learning performance \citep{bengio2009curriculum}. Although Algorithm \ref{alg:generic} does not explicitly implement curriculum learning by assigning preferences to tasks, improvement could be achieved through adaptive task selection that may reflect the benefits of curriculum learning. Intuitively, any curricula that selects tasks through an order of $M_1, \dots, M_T$ is implicitly included in Algorithm \ref{alg:generic} as it explores all the MDPs in each round with the mixture of all epsilon-greedy policies. This means that the sample-complexity of Algorithm \ref{alg:generic} provides an upper bound on the sample complexity of underlying optimal task selection. A formal discussion is deferred to Appendix \ref{app:curriculum}.

\begin{algorithm}[h]
\caption{Generic Algorithm for MTRL with Policy-Sharing}
\begin{algorithmic}[1]
\State \textbf{Input:} function class $\mathcal{F} = \mathcal{F}_1 \times \dots \times \mathcal{F}_{H+1}$, task set $\mathcal{M}$, exploration function $\expl$ %\ziping{number the lines; add some comments}
\State Initialize $\mathcal{D}_{0, M} \leftarrow \emptyset$ for all $M \in \mathcal{M}$
\For{round $t = 1, 2, \dots, \lfloor T/|\mathcal{M}|\rfloor$}
\State Offline learning oracle outputs $\hat{f}_{t, M} \leftarrow \mathcal{Q}(\mathcal{D}_{t-1, M}) %\leftarrow \hat Q(\mathcal{D}_{t-1}^M, \mathcal{F})
$ for each $M$ \Comment{Offline learning}
\State Set myopic exploration policy $\hat{\pi}_{t, M} \leftarrow \expl(\pi^{\hat{f}_{t, M}})$ for each $M$
\State Set $\hat\pi_t \leftarrow \operatorname{Mixture}(\{\hat\pi_{t, M}\}_{M \in \mathcal{M}})$ \Comment{Share policies}
\For{$M \in \mathcal{M}$}
    \State Sample one episode $\tau_{t, M}$ on MDP $M$ with policy $\hat\pi_t$ \Comment{Collect new trajectory}
    \State Add $\tau_{t, M}$ to the dataset: $\mathcal{D}_{t, M} \leftarrow  \mathcal{D}_{t-1, M} \cup \{\tau_{t, M}\}$
\EndFor
\EndFor
\State \textbf{Return} $\hat\pi_{M} = \operatorname{Mixture}(\{\hat\pi_{t, M}\}_{t \in \lfloor T/|\mathcal{M}|\rfloor})$ for each $M$
\end{algorithmic}
\label{alg:generic}
\end{algorithm}
\vspace{-10pt}

\vspace{-3mm}

\section{Generic Sample Complexity Guarantee}

In this section, we rigorously define the diversity condition and provide a sample-complexity {upper} bound for Algorithm \ref{alg:generic}. We start with introducing an intuitive example on how diversity encourages exploration in a multitask setting.
\begin{figure}
    \centering
    \includegraphics[width = 0.8\textwidth]{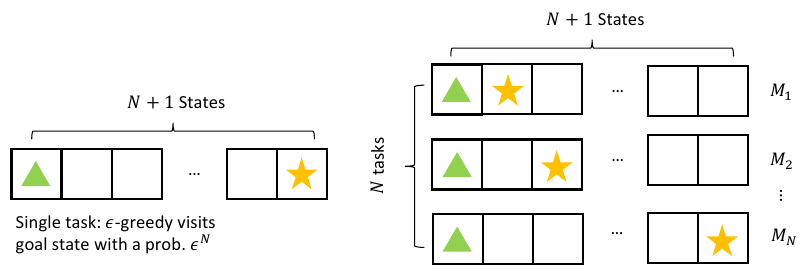}
    \vspace{-5pt}
    \caption{A diverse grid-world task set on a long hallway with $N+1$ states. From the left to the right, it represents a single-task and a multitask learning scenario, respectively. The triangles represent the starting state and the stars represent the goal states, where an agent receives a positive reward. The agent can choose to move forward or backward.}
    \label{fig:illus}
    \vspace{-10pt}
\end{figure}

{\bf Motivating example.} Figure \ref{fig:illus} introduces a motivating example of grid-world environment on a long hallway with $N+1$ states. Since this is a deterministic tabular environment, whenever a task collects an episode that visits its goal state, running an offline policy optimization algorithm with pessimism will output its optimal policy. 

Left penal of Figure \ref{fig:illus} is a single-task learning, where the goal state is $N$ steps away from the initial state, making it exponentially hard to visit the goal state with $\epsilon$-greedy exploration. Figure \ref{fig:illus} on the right demonstrates a multitask learning scenario with $N$ tasks, whose goal states "diversely" distribute along the hallway. A main message of this paper is the advantage of exploring one task by running the $\epsilon$-greedy policies of other tasks. To see this, consider any current greedy policies $(\pi_1, \pi_2, \dots, \pi_N)$. Let $i$ be the first non-optimal policy, i.e. all $j < i$, $\pi_j$ is optimal for $M_j$. Since $\pi_{i - 1}$ is optimal, by running an $\epsilon$-greedy of $\pi_{i-1}$ on MDP $M_i$, we have a probability of $\prod_{h=1}^{i-1}(1-\epsilon_h)\epsilon_i$ to visit the goal state of $M_i$, allowing it to improve its policy to optimal in the next round. Such improvement can only happen for $N$ times and all policies will be optimal within polynomial (in $N$) number of rounds if we choose $\epsilon_h = 1/(h+1)$. Hence, myopic exploration with a diverse task set leads to sample-efficient learning. The rest of the section can be seen as generalizing this idea to function approximation.

\subsection{Multitask Myopic Exploration Gap}
% The motivating example clearly relies on the fact that these MDPs share the same transition probability and their reward functions are "diverse" enough such that the disparity between the optimal policies of two adjacent tasks $M_i$ and $M_{i+1}$ are sufficiently small, allowing the $\epsilon$-greedy policy of $\pi_{M_{i}}^*$ to provide sufficient exploration for $M_{i+1}$. In this section, we formalize this idea by defining Multitask Myopic Exploration Gap and show that a diversity notation can be derived from Multitask Myopic Exploration Gap, which guarantees sample complexity of Algorithm \ref{alg:generic}.

\cite{dann2022guarantees} proposed an assumption named Myopic Exploration Gap (MEG) that allows efficient myopic exploration for a single MDP under strong assumptions on the reward function, or on the mixing time. We extend this definition to the multitask learning setting. For the conciseness of the notation, we let $\expl(f)$ denote the following mixture policy $\operatorname{Mixture}(\{\expl(\pi^{f_M})\}_{M \in \mathcal{M}})$ for a joint function $f \in (\mathcal{F})^{\otimes |\mathcal{M}|}$. Intuitively, a large myopic exploration gap implies that within all the policies that can be learned by the current exploratory policy, there exists one that can make significant improvement on the current greedy policy.

\begin{defn}[Multitask Myopic Exploration Gap (Multitask MEG)]
\label{defn:myopic_exploration_gap}
For any $\mathcal{M}$, a function class $\mathcal{F}$, a joint function $f \in (\mathcal{F})^{\otimes |\mathcal{M}|}$, %\ambuj{I don't think we're being consistent in bolding the joint $f$. I personally think bolding is a distraction. we can just say that $f$ without a subscript always denotes the joint function}, 
we say that $f$ has $\alpha(f, \mathcal{M}, \mathcal{F})$-myopic exploration gap, where $\alpha(f, \mathcal{M}, \mathcal{F})$ is the value to the following maximization problem:
$$
    \max_{M \in \mathcal{M}} \sup_{\tilde{f} \in \mathcal{F}, c \geq 1} \frac{1}{\sqrt{c}} (V_{1, M}^{\tilde{f}} - V_{1, M}^{f_{M}}), \text{ s.t. for all $f^{\prime} \in \mathcal{F}$ and $h \in [H]$,} 
$$
$$
   \begin{aligned} \mathbb{E}^{M}_{\pi^{\tilde{f}}}[(\mathcal{E}_h^M f^{\prime})(s_h, a_h)]^2 & \leq c \mathbb{E}^{ M}_{\expl(f)}[(\mathcal{E}_h^M f^{\prime})(s_h, a_h)]^2 \\ 
   \mathbb{E}^{M}_{\pi^{f_{ M}}}[(\mathcal{E}_h^M f^{\prime})(s_h, a_h)]^2 & \leq c \mathbb{E}^{ M}_{\expl(f)}[(\mathcal{E}_h^M f^{\prime})(s_h, a_h)]^2.\end{aligned}
$$
Let $M(f, \mathcal{M}, \mathcal{F})$, $c(f, \mathcal{M}, \mathcal{F})$ be the corresponding $M \in \mathcal{M}$ and $c$ that attains the maximization.
\end{defn}
\vspace{-5pt}

{\bf Design of myopic exploration gap.} To illustrate the spirit of this definition, we specialize to the tabular case, where conditions in Definition \ref{defn:myopic_exploration_gap} can be replaced by the concentrability \citep{jin2021pessimism} condition: for all $s, a, h \in \mathcal{S} \times \mathcal{A} \times [H]$, we require
\begin{equation}
\mu_{h, M}^{{\pi}^{\tilde{f}}}(s, a) \leq c \mu_{h, M}^{\expl(f)}(s, a) \text{ and } \mu_{h, M}^{\pi^{f_M}}(s, a) \leq c \mu_{h, M}^{\expl(f)}(s, a), \label{equ:concentrability}
\end{equation}
where $\mu_{h, M}^{\pi}(s, a)$ is the occupancy measure, i.e. the probability of visiting $(s, a)$ at the step $h$ by running policy $\pi$ on MDP $M$. The design of myopic exploration gap  connects deeply to the theory of offline Reinforcement Learning. For a specific MDP $M$, Equation~(\ref{equ:concentrability}) defines a set of policies with concentrability assumption \citep{xie2021bellman} that can be accurately evaluated through the offline data collected by the current behavior policy. As an extension to the Single-task MEG in \cite{dann2022guarantees}, Multitask MEG considers the maximum myopic exploration gap over a set of MDPs and the behavior policy is a mixture of all the greedy policies. Definition \ref{defn:myopic_exploration_gap} reduces to the Single-task MEG when the task set $\gM$ is a singleton.

\vspace{-3mm}

\subsection{Sample Complexity Guarantee}
We propose Diversity in Definition \ref{defn:generic_divers}, which relies on having a lower bounded Multitask MEG for any suboptimal policy. We then present Theorem \ref{thm:generic_result} that shows an upper bound for sample complexity of Algorithm \ref{alg:generic} by assuming diversity.

\begin{defn}[Multitask Suboptimality]
For a multitask RL problem with MDP set $\mathcal{M}$ and value function class $\mathcal{F}$. Let $\mathcal{F}_{\beta} \subset (\mathcal{F})^{\otimes |\mathcal{M}|}$ be the $\beta$-suboptimal class, such that for any ${f} \in \mathcal{F}_{\beta}$, there exists ${f}_M$ and $\pi^{{f}_M}$ is $\beta$-suboptimal for MDP $M$, i.e. $V_{1, {M}}^{\pi^{{f}_M}} \leq \max_{\pi \in \Pi} V_{1, {M}}^{\pi} -\beta$.
\end{defn}

\begin{defn}[Diverse Tasks]
\label{defn:generic_divers}
For some function $\tilde{\alpha}:[0, 1] \mapsto \mathbb{R}$, and $\tilde{c}:[0, 1] \mapsto \mathbb{R}$, we say that a tasks set is $(\tilde{\alpha}, \tilde{c})$-diverse if any $f \in \mathcal{F}_{\beta}$ has multitask myopic exploration gap $\alpha(f, \mathcal{M}, \mathcal{F}) \geq \tilde{\alpha}(\beta)$ and $c(f, \mathcal{M}, \mathcal{F}) \leq \tilde{c}({\beta})$ for any constant $\beta > 0$.
\end{defn}

To simplify presentation, we state the result here assuming parametric growth of the Bellman eluder dimension and covering number, that is $d_{\text{BE}}(\gF, \Pi_{\gF}, \rho) \leq d^{\text{BE}}_{\gF} \log(1/\rho)$ and $\log(N'(\gF)(\rho)) \leq d_{\gF}^{\text{cover}} \log(1/\rho)$, where $\operatorname{dim}_{\operatorname{BE}}(\mathcal{F}, \Pi_{\gF}, \rho)$ is the Bellman-Eluder dimension
of class $\mathcal{F}$ and $N'_{\mathcal{F}}(\rho) = \sum_{h = 1}^{H-1} N_{\mathcal{F}_h}(\rho) N_{\mathcal{F}_{h+1}}(\rho)$. 
A formal definition of Bellman-Eluder dimension is deferred to Appendix \ref{appendix:bellman-eluder}. This parametric growth rate holds for most of the regular classes including tabular and linear \citep{russo2013eluder}. Similar assumptions are made in \cite{chen2022statistical}.

% Details of $\rho$ can be found in Appendix \ref{app:generic_bound} and

\begin{thm}[Upper Bound for Sample Complexity]
\label{thm:generic_result}
Consider a multitask RL problem with MDP set $\mathcal{M}$ and value function class $\mathcal{F}$ such that $\mathcal{M}$ is $(\tilde{\alpha}, \tilde{c})$-diverse. Then Algorithm~\ref{alg:generic} with $\epsilon$-greedy exploration function has a sample-complexity 
$$
    \mathcal{C}(\beta, \delta) = \tilde{\mathcal{O}}\left(|\mathcal{M}|^2 H^2 d^{\text{BE}}_{\gF} d^{\text{cover}}_{\gF} \frac{\ln \tilde{c}({\beta})}{\tilde{\alpha}^2({\beta})} \ln \left(1/{\delta}\right) \right).%\footnote{We supress the poly-logarithmic factors except for $\ln(1/\delta)$ and $\ln(\tilde{c}(\beta))$}.
$$

\end{thm}

\subsection{Comparing Single-task and Multitask MEG}

The sample complexity bound in Theorem \ref{thm:generic_result} reduces to the single-task sample complexity in \cite{dann2022guarantees} when $\gM$ is a singleton. To showcase the potential benefits of multitask learning, we provide a comprehensive comparison between Single-task and Multitask MEG. We focus our discussion on $\alpha(f, \gM, \gF)$ because $c(f, \gM, \gF) \leq ((\max_{\pi} V_{1, M(f, \gM, \gF)}^{\pi} - V^{\pi^{f_M}}_{1, M(f, \gM, \gF)})/\alpha(f, \gM, \gF))^2 \leq 1/\alpha^2(f, \gM, \gF)$ and $c(f, \gM, \gF)$ only impacts sample complexity bound through a logarithmic term.

We first show that Multitask MEG is lower bounded by Single-task MEG up to a factor of $1/\sqrt{|\gM|}$. %We then demonstrate a case, where $\alpha(f, \gM, \gF) \gg \alpha(f_M, \{M\}, \gF)$.

\begin{prop}
    Let $\gM$ be any set of MDPs and $\gF$ be any function class. We have that $\alpha(f, \gM, \gF) \geq \alpha(f_M, \{M\}, \gF) / \sqrt{|\gM|}$ for all $f \in (\gF)^{\otimes |\gM|}$ and $M \in \gM$.
    \label{prop:lower_bound}
\end{prop}
Proposition \ref{prop:lower_bound} immediately implies that whenever all tasks in a task set can be learned with myopic exploration individually (see examples in \cite{dann2022guarantees}), they can also be learned with myopic exploration through MTRL with Algorithm \ref{alg:generic} with an extra factor of $|\gM|^{2}$, which is still polynomial in all the related parameters. %We remark that this extra dependence is partially due to the fact that we share policies across all MDPs, while there may be a small subset of MDPs whose exploratory policy concerns the exploration. 

We further argue that Single-task MEG can easily be exponentially small in $H$, in which case, myopic exploration fails.
\begin{prop}
    \label{prop:exp_gap}
    Let $M$ be a tabular sparse-reward MDP, i.e., $R_{h, M}(s, a) = 0$ at all $(s, a, h)$ except for a goal tuple $(s_{\text{t}}, a_{\text{t}}, h_{\text{t}})$ and $R_{h_{\text{t}}, M}(s_{\text{t}}, a_{\text{t}}) = 1$. Recall $\mu_{h, M}^{\pi}(s, a) = \E_{\pi}^M[\mathbbm{1}_{s_h = s, a_h = a}]$. Then 
    $$
        \alpha(f, \{M\}, \gF) \leq \max_{\tilde \pi}(\mu_{h_t}^{\tilde \pi}(s_t, a_t) - \mu_{h_t}^{\pi^{f}}(s_t, a_t))\sqrt{{\mu_{h_t}^{\expl(\pi^{f})}(s_t, a_t)}/{\mu_{h_t}^{\tilde{\pi}}(s_t, a_t)}} \leq \sqrt{\mu_{h_t}^{\expl(\pi^{f})}(s_t, a_t)}.
    $$
\end{prop}
\vspace{-10pt}

Sparse-reward MDP is widely studied in goal-conditioned RL, where the agent receives a non-zero reward only when it reaches a goal-state \citep{andrychowicz2017hindsight,chane2021goal}. Proposition \ref{prop:exp_gap} implies that a single-task MEG can easily be exponentially small in $H$ as long as one can find a policy $\pi$ such that its $\epsilon$-greedy version, $\expl(\pi)$, visits the goal tuple $(s_t, a_t, h_t)$ with a probability that is exponentially small in $H$. This is true when the environment requires the agent to execute a fixed sequence of actions to reach the goal tuple as it is the case in Figure \ref{fig:illus}. Indeed, in the environment described in the left panel of Figure \ref{fig:illus}, a policy that always moves left will have $\alpha(f, \{M\}, \gF) \leq \sqrt{\mu_{h_t}^{\expl(\pi^{f})}(s_t, a_t)} \leq \sqrt{\Pi_{h = 1}^H (\epsilon_{h}/2)} \leq 2^{-H/2}$. This is also consistent with our previous discussion that $\epsilon$-greedy requires $\Omega(2^{H})$ number of episodes to learn the optimal policy in the worst case. As we will show later in Section \ref{sec:tabular_case}, Multitask MEG for the tabular case can be lowered bounded by $\Omega(\sqrt{1/(|\gA||\gM|H)})$ for adequately diverse task set $\gM$, leading to an exponential separation between Single-task and Multitask MEG.

\vspace{-1.5mm}
\section{Lower Bounding Myopic Exploration Gap}
Following the generic result in Theorem \ref{thm:generic_result}, the key to the problem is to lower bound myopic exploration gap $\tilde{\alpha}({\beta})$. In this section, we lower bound $\tilde\alpha({\beta})$ for the linear MDP case. We defer an improved analysis for the tabular case and the Linear Quadratic Regulator cases to Appendix \ref{app:case_studies}. 
% See Table \ref{tab:condition_summary} for a summary of the conditions needed for lower bounding multitask myopic exploration gap.
% \begin{table}[htp]
%     \centering
%     \caption{Caption}
%     \begin{tabular}{c|c|c|c}
%         \hline
%         Settings & \# of Tasks & Diverse Tasks & Lower Bound \\
%         % \hline
%         % Tabular MDP & $\forall s, h \in \mathcal{S} \times [H], \exists M, R_{h', M}(s', a') = \mathbbm{1}[s' = s, h' = h]$  & $\sqrt{\beta^2 / (|\mathcal{M}|AH)}$ \\
        
%         \hline
%         Linear MDP & $dH$ & $\forall i, h \times [d] \times [H], \exists M$, $\theta_{h, M} = e_i$ & $\sqrt{\beta^2/ (|\mathcal{M}|AH)}$ \\ 
%         \hline
%         LQR & $d_sH$ & $\forall i, h \times [d] \times [H]$, $\exists M, R_{h, M}^s = e_i e_i^\T$ & $\sqrt{\beta^2/(d_sH)}$
%     \end{tabular}
%     \label{tab:condition_summary}
% \end{table}

% \ziping{Give details to only the linear case. Then only breifly mention the other two cases.}

% \paragraph{A More General Condition.} One may notice that the above derivation holds whenever the sparse reward MDPs at the layer $h$ provides a good coverage over the states in the $h$-step. We state this intuition as a formal statement.

% \ziping{Look into the embedding learned by deep RL. 1) see whether they are full rank. 2) if eliminate some MDPs, whether it becomes not full rank 3) see what these MDPs correspond to.}

% In this section, we study the conditions needed for lowering bounding the  multitask myopic exploration gap for Linear MDPs.

Linear MDPs have been an important case study for the theory of RL \citep{wang2019optimism,jin2020provably,chen2022improved}. It is a more general case than tabular MDP and has strong implication for Deep RL. In order to employ $\epsilon$-greedy, we consider finite action space, while the state space can be infinite.

\begin{defn}[Linear MDP \cite{jin2020provably}]
\label{defn:linear_MDP}
An MDP is called linear MDP if its transition probability and reward function admit the following form. $P_h(s' \mid s, a) = \langle \phi_h(s, a), \mu_h(s') \rangle$ for some known function $\phi_h: \mathcal{S} \times \mathcal{A} \mapsto (\mathbb{R}^{+})^d$ and unknown function $\mu_h: \mathcal{S} \mapsto (\mathbb{R}^{+})^d$. $R_h(s, a) = \langle \phi_h(s, a), \theta_h \rangle$ for unknown parameters $\theta_h$ \footnote{Note that we consider non-negative measure $\mu_h$.}. 
Without loss of generality, we assume $\|\phi_h(s, a)\| \leq 1$ for all $s, a, h \in \mathcal{S} \times \mathcal{A}\times \mathcal{H}$ and $\max \left\{\left\|{\mu}_h(s)\right\|,\left\|{\theta}_h\right\|\right\} \leq \sqrt{d}$ for all $s, h\in \mathcal{S}  \times [H]$.
\end{defn}
\vspace{-6pt}
% Note that tabular MDP is a special case of Linear MDP, if one consider $\phi_h(s, a)$ to be an one-hot vector. 
An important property of Linear MDPs is that the value function also takes the linear form and the linear function class defined below satisfies Assumption \ref{aspt:function_class}.
\begin{prop}[Proposition 2.3 \citep{jin2020provably}]
\label{prop:linear_value_function}
For linear MDPs, we have for any policy $\pi$, $Q_{h, M}^{\pi}(s, a) = \langle \phi_h(s, a), w_{h, M}^{\pi} \rangle$, where $w_{h, M}^{\pi} = \theta_{h, M} + \int_{\mathcal{S}} V_{h+1, M}^\pi(s') \mu_h(s') ds' \in \mathbb{R}^d$. Therefore, we only need to consider $\mathcal{F}_h = \{(s, a) \mapsto \langle \phi_h(s, a), w \rangle: w \in \mathbb{R}^d, \|w\|_2 \leq 2\sqrt{d}\}$.
\end{prop}

Now we are ready to define a diverse set of MDPs for the linear MDP case. %Similar to the tabular case where we sufficiently cover the state space at each step $h$, we construct a set of MDPs with unknown reward parameter $\theta_{h}$ that spans the whole $\mathbb{R}^d$ space at each step $h$.

\begin{defn}[Diverse MDPs for linear MDP case]
\label{def:linear_diverse}
We say $\mathcal{M}$ is a diverse set of MDPs for the linear MDP case, if they share the same feature extractor $\phi_h$ and the same measure $\mu_h$ (leading to the same transition probabilities) and for any $h \in [H]$, there exists a subset $\{M_{i, h}\}_{i \in [d]} \subset \mathcal{M}$, 
such that the reward parameter $\theta_{h, M_{i, h}} = \1_i$ and all the other $\theta_{h', M_{i, h}} = 0$ with $h' \neq h$, where $\1_i$ is the onehot vector with a positive entry at the dimension $i$. 
% \footnote{Note that this diversity definition is quite restricted even for the linear MDPs. We discuss a potential attempt to extend this definition and its technical challenge in Appendix \ref{appendix:bellman-eluder}. %\ambuj{this is misleading since what is actually discussed in the appendix is that a more general definition will *not* work}}
\end{defn}
\vspace{-6pt}

% \ziping{Say that we have some detailed discussions on more general tasks set.}

We need the assumption that the minimum eigenvalue of the covariance matrix is strictly lower bounded away from 0. The feature coverage assumption is commonly use in the literature that studies Linear MDPs \citep{agarwal2022provable}. Suppose Assumption \ref{aspt:linear_cover} hold, we have Theorem \ref{thm:linear}, which lower bounds the multitask myopic exploration gap. Combined with Theorem \ref{thm:generic_result}, we have a sample-complexity bound of $\tilde{\mathcal{O}}(|\mathcal{M}|^3 H^3 d^2 |\mathcal{A}| /(\beta^2 b_1^2))$ with $|\mathcal{M}| \geq d$.

\begin{aspt}[Feature coverage]
For any $\nu \in \mathbb{S}^{d-1}$ and $[\nu]_i > 0$ for all $i \in [d]$, there exists a policy $\pi$ such that 
$
    \mathbb{E}_{\pi}[\nu^\top\phi_{h}(s_{h}, a_{h})] \geq b_1,
$
for some constant $b_1 > 0$.
\label{aspt:linear_cover}
\end{aspt}
\begin{thm}
\label{thm:linear}
Consider $\mathcal{M}$ to be a diverse set as in Definition \ref{def:linear_diverse}. Suppose Assumption \ref{aspt:linear_cover} holds and $\beta \leq b_1/2$, then we have for any ${f} \in \mathcal{F}_{\beta}$,
$
    \alpha({f}, \mathcal{F}, \mathcal{M}) = \Omega(\sqrt{{\beta^2b_1^2}/(|\mathcal{A}||\mathcal{M}|H)}
$ by setting $\epsilon_h = 1/(h+1)$.
\end{thm}
\vspace{-5pt}

Proof of Theorem 
\ref{thm:linear} critically relies on iteratively showing the following lemma, which can be interpreted as showing that the feature covariance matrix induced by the optimal policies is full rank.
\begin{lem}
\label{lem:linear_2}
Fix a step $h$ and fix a $\beta < b_1 /2$. Let $\{\pi_i\}_{i=1}^d$ be $d$ policies such that $\pi_i$ is a $\beta$-optimal policy for $M_{i, h}$ as in Definition \ref{def:linear_diverse}. Let $\tilde{\pi} = \operatorname{Mixture}(\{\expl(\pi_i)\}_{i = 1}^d)$. Then we have
$
    \lambda_{\operatorname{min}}(\Phi_{h+1}^{\tilde{\pi}})\geq {\epsilon_h\prod_{h'=1}^{h-1}(1-\epsilon_{h'}) b_1^2}/{(2dA)}.
$
\end{lem}

% \paragraph{Connections to curriculum learning.} Note that the proof reflects an interesting connection to curriculum learning, as we show that the near-optimal policies at the step $h$ lead to a good coverage at the step $h+1$, which allows the algorithm to learn the optimal policies at the step $h+1$. The proof is proceeded as if it follows a curriculum from smaller steps to larger steps.

% The proof is completed by induction. Starting with $h = 1$, random policy leads to a full-rank covariance matrix, which guarantees a good exploration for the current step $h$. The algorithm may learn the optimal policy for MDP corresponding to the step $h$, which provides full-rank covariance matrix for the step $h+1$.
\vspace{-3mm}
\subsection{Discussions on the Tabular Case}
\vspace{-6pt}
\label{sec:tabular_case}
Diverse tasks in Definition~\ref{def:linear_diverse}, when specialized to the tabular case, corresponds to $S\times H$ sparse-reward MDPs. Interestingly, similar constructions are used in reward-free exploration \citep{jin2020reward}, which shows that by calling an online-learning oracle individually for all the sparse reward MDPs, one can generate a dataset that outputs a near-optimal policy for any given reward function. We want to point out the intrinsic connection between the two settings: our algorithm, instead of generating an offline dataset all at once, generates an offline dataset at each step $h$ that is sufficient to learn a near-optimal policy for MDPs that corresponds to the step $h+1$.

{\bf Relaxing coverage assumption.} Though feature coverage Assumption (Assumption \ref{aspt:linear_cover}) is handy for our proof as it guarantees that any $\beta$-optimal policy (with $\beta < b_1/2$) has a probability at least $b_1/2$ to visit their goal state, this assumption may not be reasonable for the tabular MDP case. Intuitively, without this assumption, a $\beta$-optimal policy can be an arbitrary policy and we can have at most $S$ such policies in total leading to a cumulative error of $S\beta$. A naive solution is to request a $S^{-H} \beta$ accuracy at the first step, which leads to exponential sample-complexity. In Appendix \ref{app:remove_aspt2}, we show that an improved analysis can be done for the tabular MDP without having to make the coverage assumption. However, an extra price of $SH$ has to be paid.

% \begin{lem}
% Fix a step $h$. Let $\{\pi_i\}_{i = 1}^d$ be $d$ policies such that $\pi_i$ is a $b_3/2$-optimal policy for LQR with reward $R_h^s = e_i e_i^\T$ and $R_h^a = I$, such that $r_h(s, a) = s_{[i]}^2 - \|a\|_2^2$. Let $\tilde{\pi}$ be a mixture of $\{\expl(\pi_i)\}_{i = 1}^{d_s}$. Then for any $\nu \in \mathbb{S}^{d-1}$, we have
% $$
%     \mathbb{E}_{\tilde{\pi}} [\nu^\T s_{h+1} s_{h+1}^\T \nu] \geq ...
% $$
% \end{lem}
% \begin{proof}

% \end{proof}

\section{Implications of Diversity on Robotic Control Environments}
\label{app:experiment_full_details}
In this section, we conduct simulation studies on robotic control environments with practical interests. Since myopic exploration has been shown empirically efficient in many problems of interest, we focus on the other main topic--diversity.
We investigate how our theory guides a diverse task set selection. More specifically, our prior analysis on Linear MDPs suggests that a diverse task set should prioritize tasks with full-rank feature covariance matrices. We ask whether tasks with a more spread spectrum of the feature covariance matrix lead to a better training task set. \textit{Note that the goal of this experiment is not to show the practical interests of Algorithm~\ref{alg:generic}. Instead, we are revealing interesting implications of the highly conceptual definition of diversity in problems with practical interests.} %\ambuj{we should prolly say sth to justify this decision. I think the justification goes sth like this: myopic greedy in MTRL settings has already been shown to work well empirically. Since we're explaining sth that is already known to occur, it makes little sense to repeat similar empirical work. instead, we chose to focus on novel implications coming our of the theoretical analysis}

% \zifan{In this section, we conduct simulation studies on robotic control environment with practical interests. Our aim is to explore how our theory guides the selection of task sets. Our prior analysis on linear MDP suggests that an effective curriculum should prioritize tasks with full-rank covariance matrices for embeddings. To verify this hypothesis, we investigate how the tasks selected by a well-established ACL algorithm align with this criterion. Furthermore, we investigate the efficacy of multiple simple two-step curricula and demonstrate that the best performance is achieved when selecting tasks that adhere to our proposed task selection rule.}

% \subsection{Simulations on Tabular MDPs}
%The environment is controlled by two parameters: the stumps height and the spacing between stumps (see Figure \ref{fig:env}). Intuitively, a higher stump and a smaller spacing between stumps requires better policies and some spacing and height combinations are even infeasible. We study heights  $p \in [0, 3.0]$ and spacing $q \in [0, 6.0]$ and each environment with parameter $p, q$ is denoted by $M_{p, q}$. Note that unlike the case we discussed above, the task are implicitly encoded in the state space an agent can reach. For example, in a task with no stump, the agent can never observe stumps with its lidar sensor.
{\bf Environment and training setup.} We adopt the BipedalWalker environment from \citep{portelas2020teacher}. The learning agent is embodied into a bipedal walker whose motors are controllable with torque (i.e. continuous action space). The observation space consists of laser scan, head position, and joint positions. The objective of the agent is to move forward as far as possible, while crossing stumps with varying heights at regular intervals (see Figure \ref{fig:cov_mat_full} (a)). The agent receives positive rewards for moving forward and negative rewards for torque usage. An environment or task, denoted as $M_{p,q}$, is controlled by a parameter vector $(p, q)$, where $p$ and $q$ denote the heights of the stumps and the spacings between the stumps, respectively. % The value $p$ represents the mean of a normal distribution $\mathcal{N}(p, 0.1)$, from which the heights of the stumps are sampled, and the value $q$ determines the spacing between the stumps. 
Intuitively, an environment with higher and denser stumps is more challenging to solve. We set the parameter ranges for $p$ and $q$ as $p \in [0, 3]$ and $q \in [0, 6]$ in this study. %Note that, unlike the case discussed above where the task is explicitly provided as input to the agent, in this case, the task is implicitly encoded by the agent's observation of the environment.

% \begin{figure}[htp]
%     \centering
%     \includegraphics{Figure/env.jpg}
%     \caption{Stump Tracks, a parametric environment for BipedalWalker.}
%     \label{fig:env}
% \end{figure}

The agent is trained by Proximal Policy Optimization (PPO) \citep{schulman2017proximal} with a standard actor-critic framework \citep{konda1999actor} and with Boltzmann exploration that regularizes entropy. Note that Boltzmann exploration strategy is another example of myopic exploration, which is commonly used for continuous action space. Though it deviates from the $\epsilon$-greedy strategy discussed in the theoretical framework, we remark that the theoretical guarantee outlined in this paper can be trivially extend to Boltzmann exploration. The architecture for the actor and critic feature extractors consists of two layers with 400 and 300 neurons, respectively, and Tanh \citep{rumelhart1986learning} as the activation function. Fully-connected layers are then used to compute the action and value. We keep the hyper-parameters for training the agent the same as \citet{romac2021teachmyagent}.

% \ziping{We use epsilon greedy exploration}

\subsection{Investigating Feature Covariance Matrix} 
% As discussed in our theory for Linear MDPs, a diverse set of tasks will encourage a full-rank structure on the feature covariance matrix. In practice, the optimal embeddings are not known, neither is the environment a linear MDP. %As we will see later that the matrix is actually low-rank. 
% However, we conjecture that a good set of tasks will encourage the spectrum of the covariance matrix to be more spread than the others. 
% \zifan{A good set of task sounds like a very weak argument. Since we are using an ACL algorithm to select task, would it be useful if we evaluate on the tasks that the teacher actually selected, and show that the specturm is more spread than random? Use some numbers to show, instad of just describing.}

%\zifan{In "Training details", we have already introduced the RL algorithm and NN architectures. So here could you just referring to "the output of the feature extractor is denoted by...". Also this training is using ALP-GMM}
We denote by  $\phi(s, a)$ the output of the feature extractor. % Since the optimal embeddings are unknown, we compute the embeddings using the feature extractor of the actor at the end of training, denoted as $\phi{(s, a)}$. 
We evaluate the extracted feature at the end of the training generated by near-optimal policies, denoted as $\pi$, on 100 tasks with different parameter vectors $(p, q)$. We then compute the covariance matrix of the features for each task, denoted as $V_{p, q} = \mathbb{E}_{\pi}^{M_{p, q}} \sum_{h = 1}^H \phi(s_h, a_h) \phi{(s_h, a_h)}^T$, whose spectrums are shown in Figure \ref{fig:cov_mat_full} (b) and (c).

% We investigate a trained agent, with the optimal feature extractor and the near optimal policies for different environments denoted by $\pi_{p, q}$, where $p$ and $q$ are the height parameter and the spacing parameter respectively. Let 
% $V_{p, q} = \mathbb{E}^{\pi_{p, q}}_{M_{p, q}} \sum_{h = 1}^H \phi_K(s_h, a_h) \phi_K(s_h, a_h)^T.$
% We investigate the spectrum of $V_{p, q}$ in Figure \ref{fig:cov_mat_full}. 
By ignoring the extremely large and small eigenvalues on two ends, we focus on the largest 100-200 dimension, where we observe that the height of the stumps $p$ has a larger impact on the distribution of eigenvalues. In Figure \ref{fig:cov_mat_full} (b), we show the boxplot of the log-scaled eigenvalues of 100-200 dimensions for environments with different heights, where we average spacings. We observe that the eigenvalues can be 10 times higher for environments with an appropriate height (1.0-2.3), compared to extremely high and low heights. However, the scales of eigenvalues are roughly the same if we control the spacings and take average over different heights as shown in Figure \ref{fig:cov_mat_full} (c). This indicates that choosing an appropriate height is the key to properly scheduling tasks. 

In fact, the task selection coincidences with the tasks selected by the state-of-the-art Automatic Curriculum Learning (ACL). We investigate the curricula generated by ALP-GMM \citep{portelas2020teacher}, a well-established curriculum learning algorithm, for training an agent in the BipedalWalker environment for 20 million timesteps. Figure \ref{fig:cov_mat_full} (d) gives the density plots of the ACL task sampler during the training process, which shows a significant preference over heights in the middle range, with little preference over spacing.

\begin{figure}[bht!]
    \centering
    \begin{subfigure}[b]{0.48\textwidth}
    \centering
    \includegraphics[width = 0.8\textwidth]{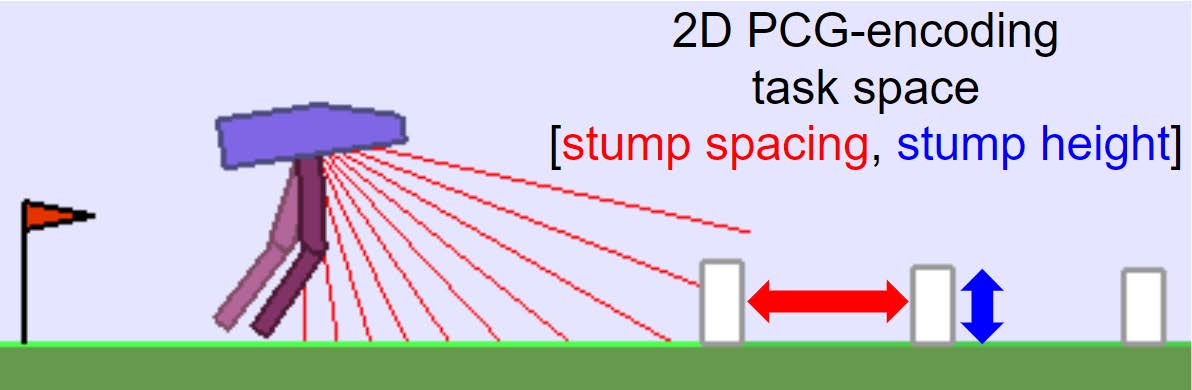}
    \caption{BipedalWalker environment}
    \end{subfigure}
    \begin{subfigure}[b]{0.48\textwidth}
    \includegraphics[width = 1\textwidth]{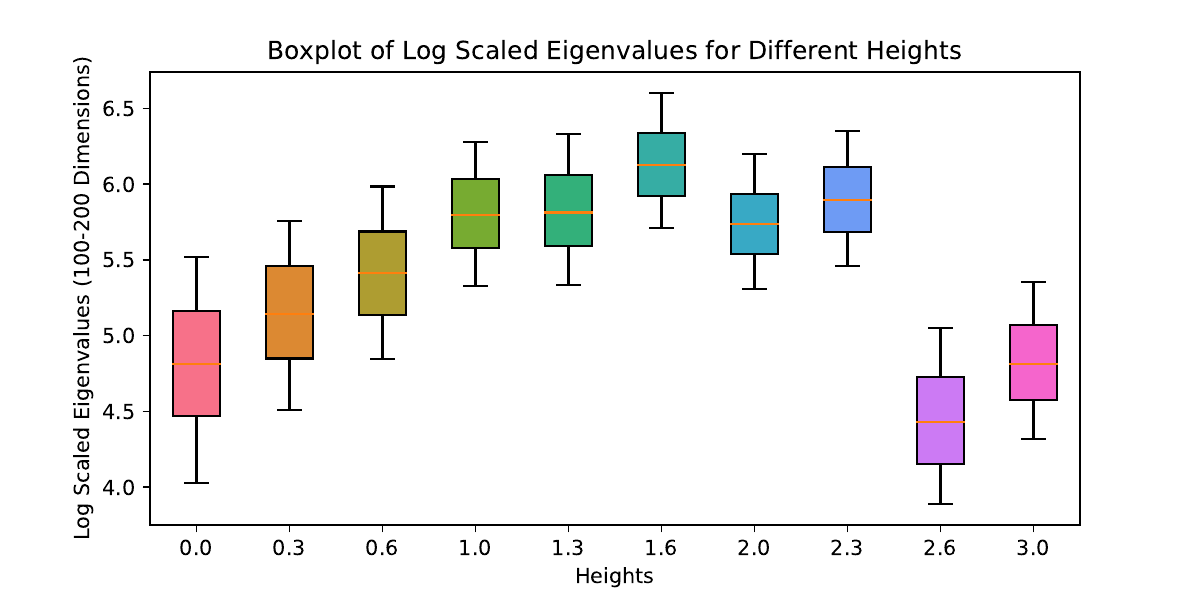}
    \caption{Controlling heights}
    \end{subfigure}
    % \begin{subfigure}[b]{0.33\textwidth}
    % \includegraphics[width = 1\textwidth]{Figure/Bipedal/heights_zoom_in.pdf}
    % \end{subfigure}
    \begin{subfigure}[b]{0.48\textwidth}
    \includegraphics[width = 1\textwidth]{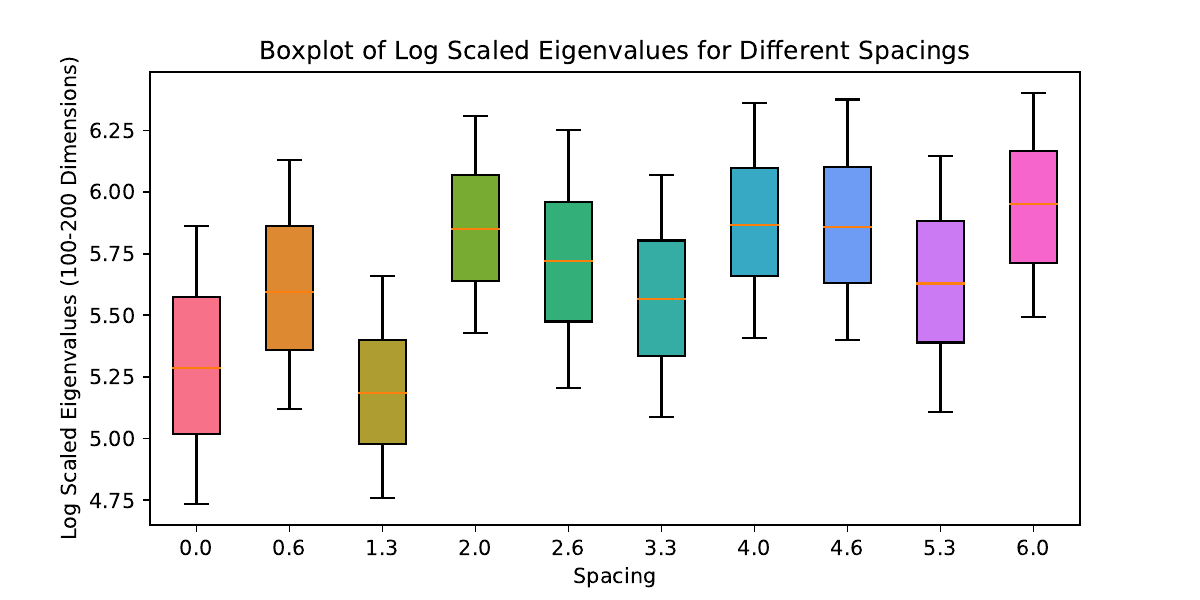}
    \caption{Controlling spacings}
    \end{subfigure}
    % \begin{subfigure}[b]{0.48\textwidth}
    % \includegraphics[width = 1\textwidth]{Figure/Bipedal/heights_zoom_in.pdf}
    % \caption{Controlling heights in the original scale}
    % \end{subfigure}
    \begin{subfigure}[b]{0.48\textwidth}
    \centering
    \includegraphics[width = 0.48\textwidth, height = 0.48\textwidth]{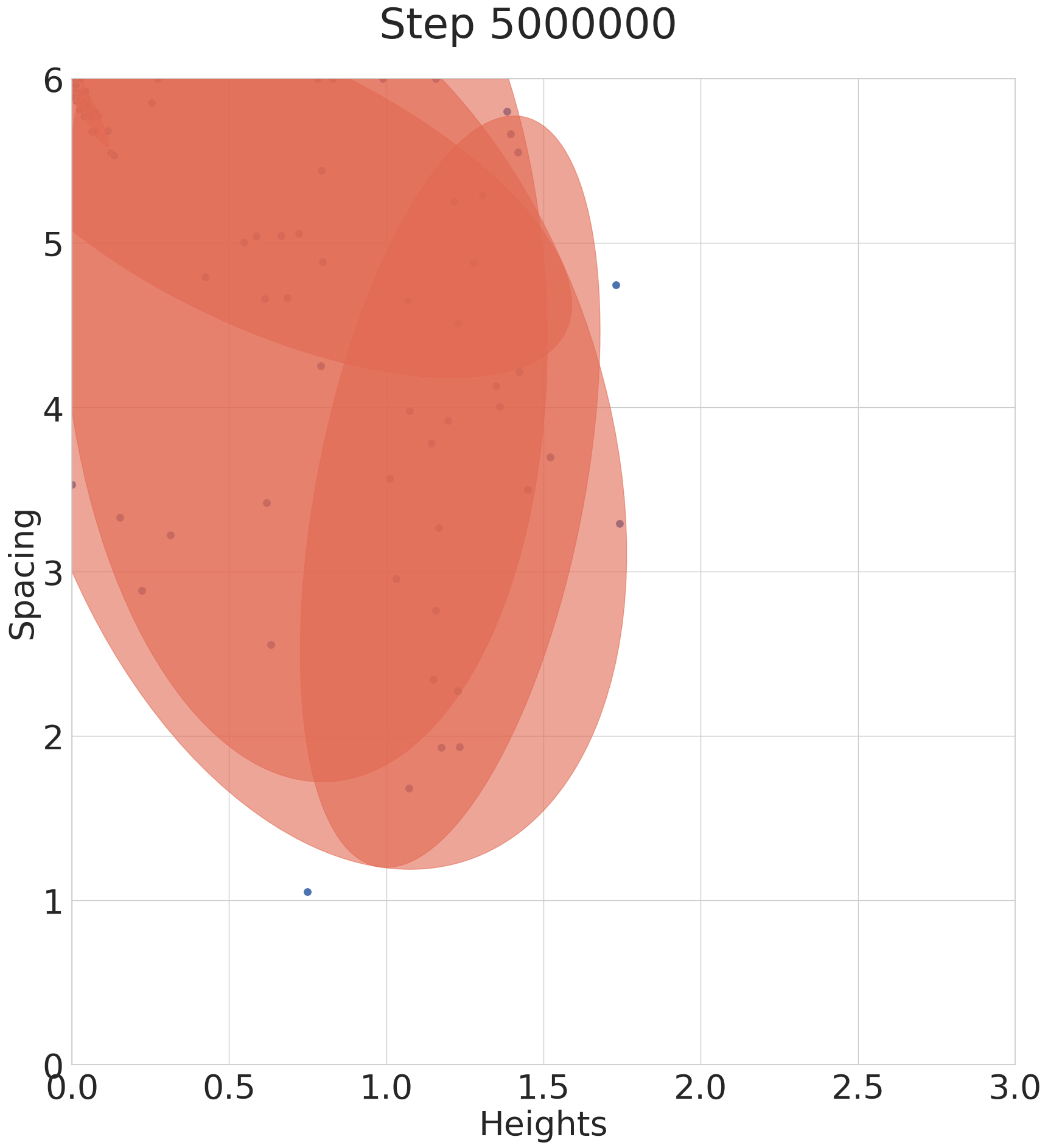}
    \includegraphics[width = 0.48\textwidth, height = 0.48\textwidth]{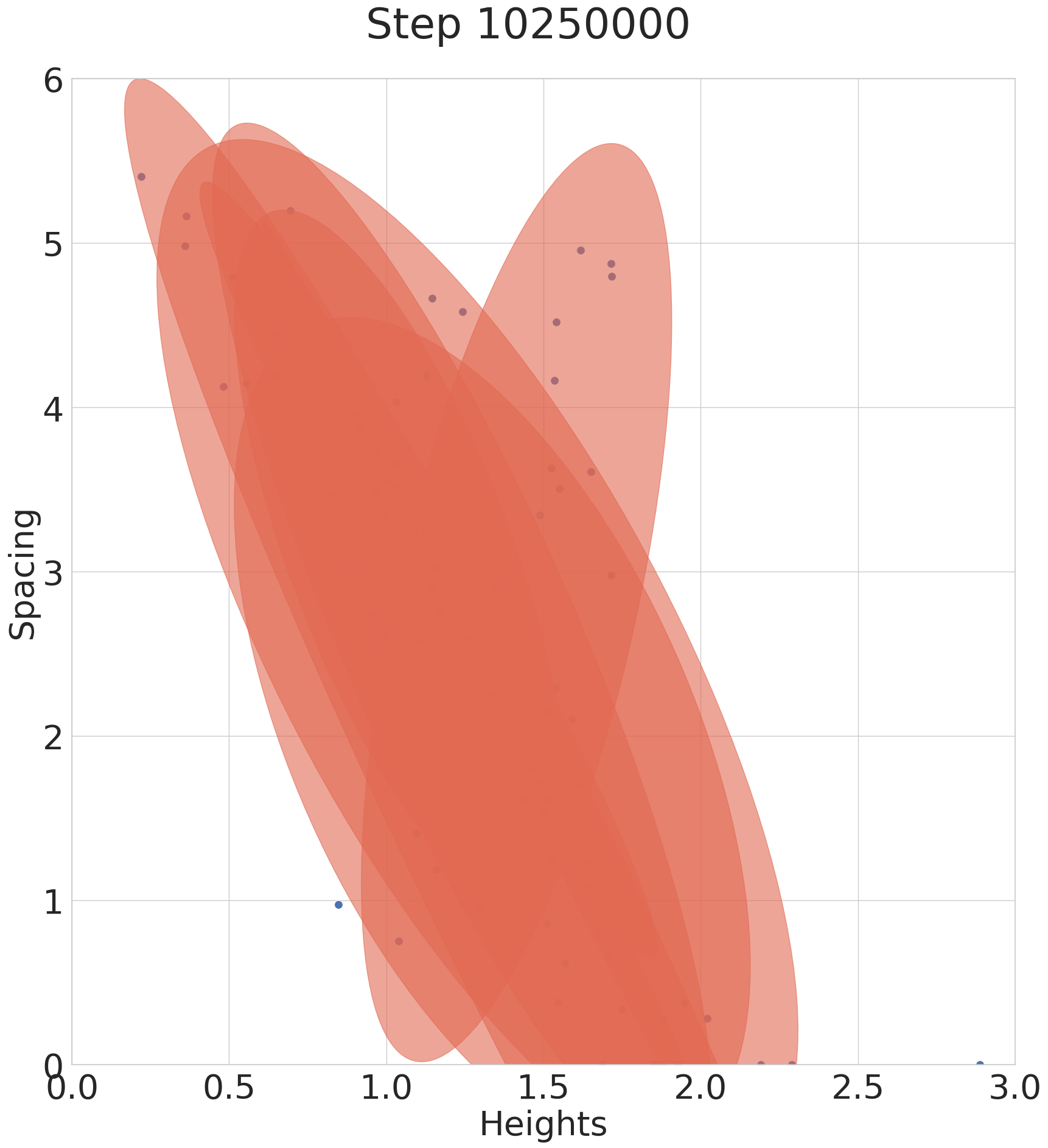}
    \caption{Preference of automatic CL}
    \end{subfigure}
    \caption{{\bf (a)} BipedalWalker Environment with different stump spacing and heights. {\bf (b-c)} Boxplots of the log-scaled eigenvalues of sample covariance matrices of the trained embeddings generated by the near optimal policies for different environments. In (b), we take average over environments with the same height while in (c), over the same spacing. {\bf (d)} Task preference of automatically generated curriculum at 5M and 10M training steps respectively. The red regions are the regions where a task has a higher probability to be sampled.}
    \label{fig:cov_mat_full}
\end{figure}

{\bf Training on different parameters.}
To further validate our finding, we train the same PPO agent with different means of the stump heights and see that how many tasks does the current agent master. As we argued in the theory, a diverse set of tasks provides good behavior policies for other tasks of interest. Therefore, we also test how many tasks it could further master if one use the current policy as behavior policy for fine-tuning on all tasks. The number of tasks the agent can master by learning on environments with heights ranging in [0.0, 0.3], [1.3, 1.6], [2.6, 3.0] are 28.1, 41.6, 11.5, respectively leading to a significant outperforming for diverse tasks ranging in [1.3, 1.6]. Table \ref{tab:1} gives a complete summary of the results.

\begin{table}[htp]
\caption{Training on different environment parameters. Each row represents a training scenario, where the first two columns are the range of sampled parameters. The mastered tasks are out of 121 evaluated tasks with the standard deviation calculated from ten independent runs.}
\label{tab:1}
\centering
\begin{tabular}{|l|l|l|l|}
\hline
Obstacle spacing & Stump height & Mastered task\\ \hline
{[}2, 4{]}       & {[}0.0, 0.3{]}                      & $28.1 \pm 6.1$                    \\ \hline
{[}2, 4{]}       & {[}1.3, 1.6{]}                       & $41.6 \pm 9.8$                    \\ \hline
{[}2, 4{]}       & {[}2.6, 3.0{]}                       & $11.5 \pm 10.9$                    \\ \hline
\end{tabular}
\end{table}

\begin{figure}
    \centering
    \begin{subfigure}[b]{0.48\textwidth}
    \includegraphics[width = 1\textwidth]{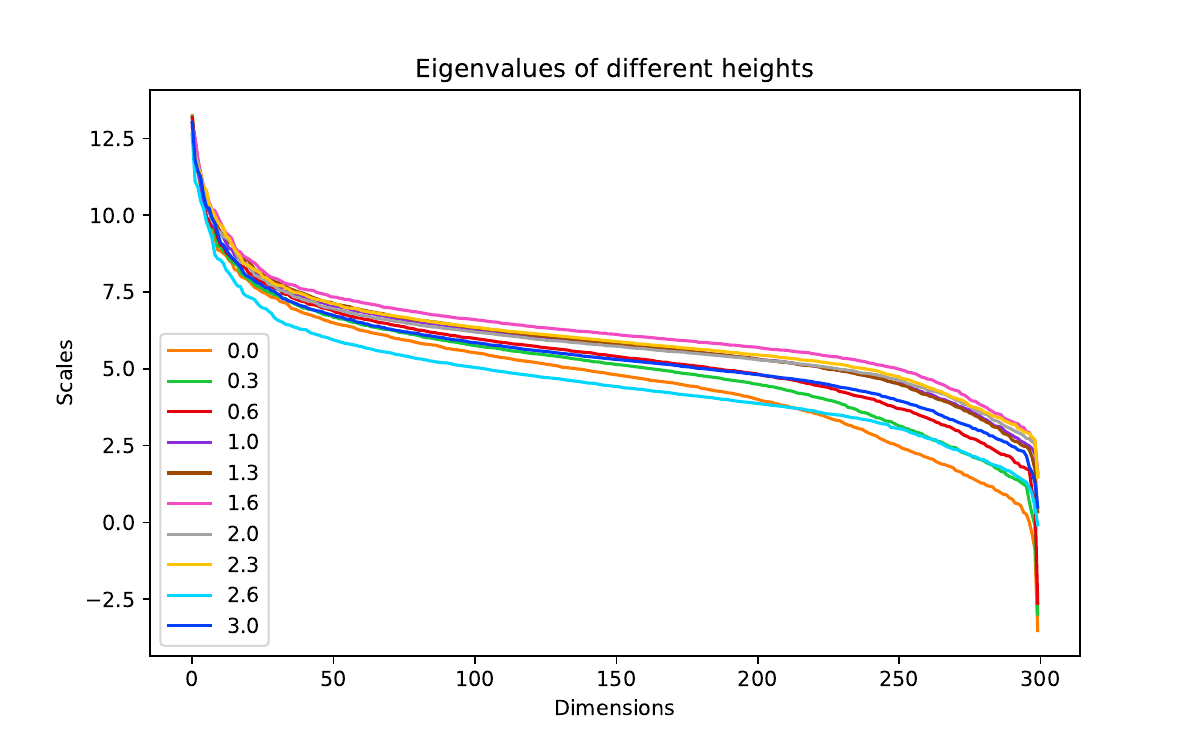}
    \caption{Controlling heights}
    \end{subfigure}
    \begin{subfigure}[b]{0.48\textwidth}
    \includegraphics[width = 1\textwidth]{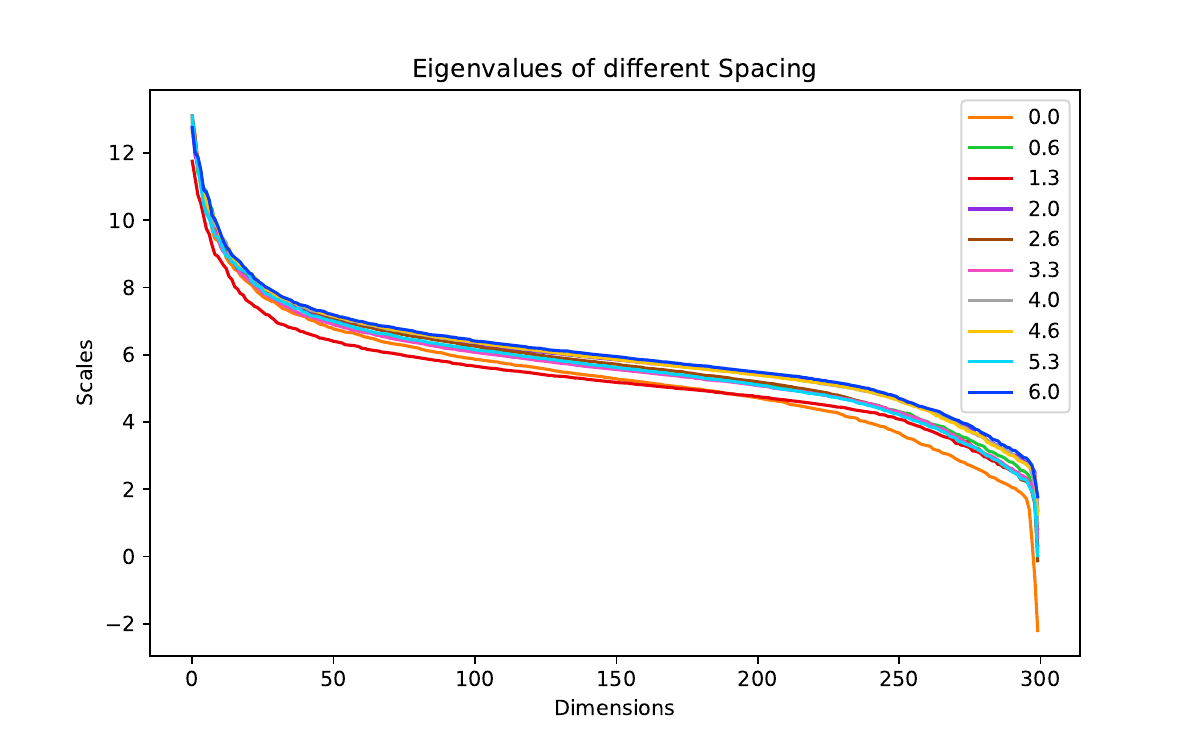}
    \caption{Controlling spacing}
    \end{subfigure}
    \begin{subfigure}[b]{0.48\textwidth}
    \includegraphics[width = 1\textwidth]{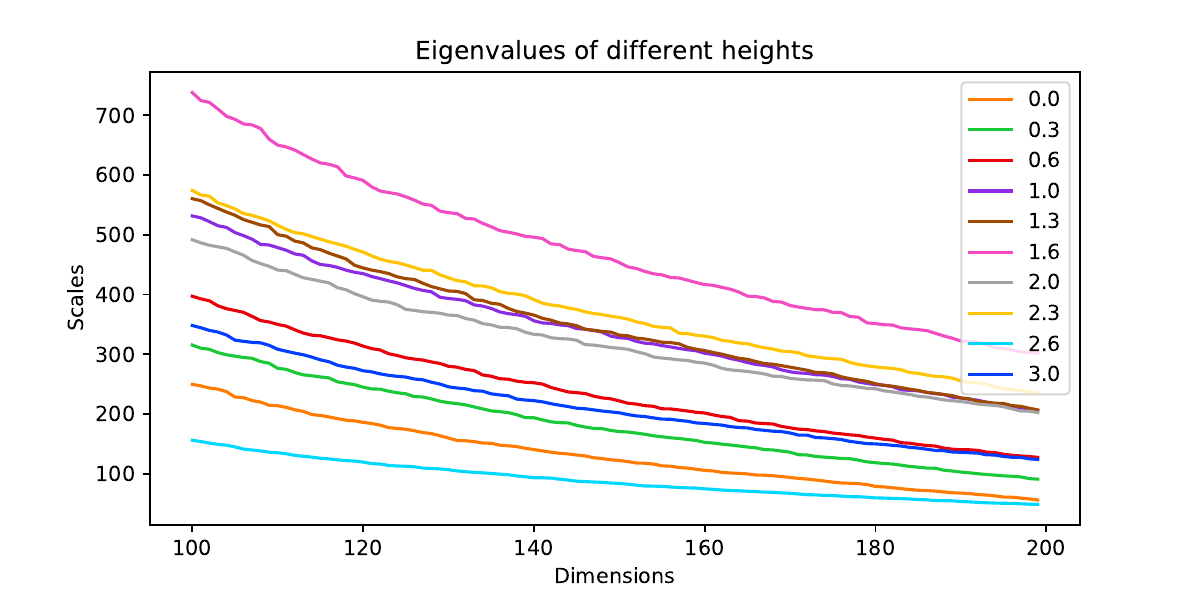}
    \caption{Controlling heights on the original scale}
    \end{subfigure}
    \caption{{\bf (b-c)} Log-scaled eigenvalues of sample covariance matrices of the trained embeddings generated by the near optimal policies for different environments.}
    \label{fig:full_cov_mat}
\end{figure}

\iffalse
\begin{table}[htp]
\caption{Training on different environment parameters. Each row represents a training scenario, where the first two columns are the range of sampled parameters.}
\label{tab:1}
\centering
\begin{tabular}{|l|l|l|l|}
\hline
Obstacle spacing & Stump height & Mastered task (before) & Mastered task (after) \\ \hline
{[}0, 6{]}       & {[}0, 1.0{]} &                     & $26.4 \pm 2.7$                    \\ \hline
{[}0, 6{]}       & {[}1.0, 2.0{]} &                      & $33.8 \pm 13.0$                    \\ \hline
{[}0, 6{]}       & {[}2.0, 3.0{]} &                      & $11.4 \pm 8.1$                    \\ \hline
\end{tabular}
\end{table}
\fi

% \ziping{We still need that a $\beta$-suboptimal policy leads to a similar distribution of the optimal one.}

\newpage

\section{Acknowledgement}
Prof. Tewari acknowledges the support of NSF via grant IIS-2007055. Prof. Stone runs the Learning Agents Research Group (LARG) at UT
Austin.  LARG research is supported in part by NSF (FAIN-2019844,
NRT-2125858), ONR (N00014-18-2243), ARO (E2061621), Bosch, Lockheed
Martin, and UT Austin's Good Systems grand challenge.  Peter Stone
serves as the Executive Director of Sony AI America and receives
financial compensation for this work.  The terms of this arrangement
have been reviewed and approved by the University of Texas at Austin in
accordance with its policy on objectivity in research. 

Part of this work is completed when Ziping is a postdoc at Harvard University. He is supported by Prof. Susan A. Murphy through NIH/NIDA P50DA054039, NIH/NIBIB, OD P41EB028242 and NIH/NIDCR UH3DE028723.

\newpage

\appendix
\section{Related Works}

\paragraph{Multitask RL.} Many recent theoretical works have contributed to understanding the benefits of MTRL \citep{agarwal2022provable,brunskill2013sample,calandriello2014sparse,cheng2022provable, lu2021power,uehara2021representation,yang2022nearly,zhang2021provably} by exploiting the shared structures across tasks. An earlier line of works \citep{brunskill2013sample} assumes that tasks are clustered and the algorithm adaptively learns the identity of each task, which allows it to pool observations. For linear Markov Decision Process (MDP) settings \citep{jin2020provably}, Lu et al. \citep{lu2021power} shows a bound on the sub-optimality of the learned policy by assuming a full-rank least-square value iteration weight matrix from source tasks. Agarwal et al. \citep{agarwal2022provable} makes a different assumption that the target transition probability is a linear combination of the source ones, and the feature extractor is shared by all the tasks. Our work differs from all these works as we focus on the reduced complexity of exploration design.

\paragraph{Curriculum learning.} Curriculum learning refers to adaptively selecting tasks in a specific order to improve the learning performance \citep{bengio2009curriculum} under a multitask learning setting. Numerous studies have demonstrated improved performance in different applications \citep{jiang2015self,pentina2015curriculum,graves2017automated,wang2021survey}. However, theoretical understanding of curriculum learning remains limited. \cite{xu2022statistical} study the statistical benefits of curriculum learning under Supervised Learning setting. For RL, \cite{li2022understanding} makes a first step towards the understanding of sample complexity gains of curriculum learning without an explicit exploration bonus, which is a similar statement as we make in this paper. However, their results are under strong assumptions, such as prior knowledge on the curriculum and a specific contextual RL setting with Lipschitz reward functions. This work can be seen as a more comprehensive framework of such benefits, where we discuss general MDPs with function approximation.

%\paragraph{Contributions.} We summarize the main contributions of this work in the following aspects.

\paragraph{Myopic exploration.} Myopic exploration, characterized by its ease of implementation and effectiveness in many problems \citep{kalashnikov2018scalable,mnih2015human}, is the most commonly used exploration strategy. Many theory works \citep{dabney2020temporally,dann2022guarantees,liu2018simple,simchowitz2020naive} have discussed the conditions, under which myopic exploration is efficient. However, all these studies consider a single MDP and require strong conditions on the underlying environment. Our paper closely follows \cite{dann2022guarantees} where they define Myopic Exploration Gap. An MDP with low Myopic Exploration Gap can be efficiently learned by exploration exploration.

\section{A Formal Discussion on Curriculum Learning}
\label{app:curriculum}
We formally discuss that how our theory could provide a potential explanation on the success of curriculum learning in RL \citep{narvekar2020curriculum}. Although Algorithm \ref{alg:generic} does not explicitly implement curriculum learning by ordering tasks, we argue that if any curriculum learn leads to polynomial sample complexity $\gC$, then Algorithm \ref{alg:generic} has $|\gM|^2 \gC$ sample complexity. We denote a curricula by $((M_i, T_i))_{i = 1}^T$ and an online algorithm that learns through the curricula interacts with $M_i$ for $T_i$ rounds by rolling out trajectories with the estimated optimal policy of $M_{i-1}$ with epsilon greedy. This curricula is implicitly included in Algorithm \ref{alg:generic} with $|\gM|\sum_{i} T_i$ rounds. To see this, let us say in phase $i$, the algorithm has mastered all tasks $M_1, \dots, M_{i-1}$. Then by running Algorithm \ref{alg:generic} $|\gM|^2 T_i$ rounds, we will roll out $T_i$ trajectories on $M_i$ using the exploratory policy from $M_{i-1}$ on average, which reflects the schedules from curricula. This means that the sample-complexity of Algorithm \ref{alg:generic} provides an upper bound on the sample complexity of underlying optimal curricula and in this way our theory provides some insights on the success of the curriculum learning.

\section{Comparing Single-task MEG and Multitask MEG}

{\bf Proposition \ref{prop:lower_bound}.}{\it\quad Let $\gM$ be any set of MDPs and $\gF$ be any function class. We have that $\alpha(f, \gM, \gF) \geq \alpha(f_M, \{M\}, \gF) / \sqrt{|\gM|}$ for all $f \in (\gF)^{\otimes |\gM|}$ and $M \in \gM$.}
\begin{proof}
    The proof is straightforward from the definition of multitask MEG. For any MDP $M$ and any $f \in \gF^{\otimes |\gM|}$, $\alpha(f_M, \{M\}, \gF)$ is the value to the following optimization problem
    $$
        \sup_{\tilde{f} \in \mathcal{F}, c \geq 1} \frac{1}{\sqrt{c}} (V_{1, M}^{\tilde{f}} - V_{1, M}^{f_{M}}), \text{ s.t. for all $f^{'} \in \mathcal{F}$ and $h \in [H]$,} 
    $$
    \begin{align*}
        \mathbb{E}^{M}_{\pi^{\tilde{f}}}[(\mathcal{E}_h^M f^{\prime})(s_h, a_h)]^2 & \leq c \mathbb{E}^{ M}_{\expl(f)}[(\mathcal{E}_h^M f^{\prime})(s_h, a_h)]^2 \\ 
   \mathbb{E}^{M}_{\pi^{f_{ M}}}[(\mathcal{E}_h^M f^{\prime})(s_h, a_h)]^2 & \leq c \mathbb{E}^{ M}_{\expl(f)}[(\mathcal{E}_h^M f^{\prime})(s_h, a_h)]^2.
    \end{align*}
    By choosing $c$ in Definition \ref{defn:myopic_exploration_gap} by $c |\gM|$, and $f'$ by the same $f'$ that attains the maximization in Single-task MEG, we have $\alpha(f, \gM, \gF) \geq \alpha(f_M, \{M\}, \gF) / \sqrt{|\gM|}$
\end{proof}

\section{Efficient Myopic Exploration for Deterministic MDP with known Curriculum}
\label{app:CL_example}

In light of the intrinsic connection between Algorithm \ref{alg:generic} and curriculum learning. We present an interesting results for curriculum learning showing that any deterministic MDP can be efficiently  learned through myopic exploration when a proper curriculum is given.

\begin{prop}
For any deterministic MDP $M$, with sparse reward, there exists a sequence of deterministic MDPs $M_1, M_2, \dots, M_H$, such that the following learning process returns a optimal policy for $M$:
\begin{enumerate}
    \item Initialize $\pi_0$ by a random policy.
    \item For $t = 1, \dots, n$, follow $\pi_{t-1}$ with an $\epsilon$-greedy exploration to collect $4At \log(H/\delta)$ trajectories denoted by $\mathcal{D}_t$. Compute the optimal policy $\pi_t$ from the model learned by $\mathcal{D}_t$.
    \item Output $\pi_H$.
\end{enumerate}
The above procedure will end in $O(AH^2\log(H/\delta))$ episodes and with a probability at least $1-\delta$, $\pi_H$ is the optimal policy for $M$.
\end{prop}

\begin{proof}
We construct the sequence in the following manner. Let the optimal policy for an MDP $M$ be $\pi^*_{M}$. Let the trajectory induced by $\pi^*_{M}$ be $\{s^*_0, a^*_0, \dots, s^*_H, a^*_H\}$. The MDP $M$ receives a positive reward only when it reaches $s^*_{H}$. Without loss of generality, we assume that $M$ is initialized at a fixed state $s_0$. We choose $M_{n}$ such that 
$$
    R_{M_i}(s, a) = \mathbbm{1}(P_{M_n}(s, a) = s^*_{i}).
$$

Furthermore, we set 
\[P_{M_i}(s_i^* | s_i^*, a) = 1 \; \forall a \in \mathcal{A}\]
and 
\[P_{M_i} = P_{M}\]
otherwise.

This ensures that any policy that reaches $s_i^*$ on the $i$'th step is an optimal policy.

We first provide an upper bound on the expected number of episodes for finding an optimal policy using the above algorithm for $M_i$.

Fix $2 \leq i \leq H$. Let $\epsilon = \frac{1}{i}$. Define $k = |A|$. Then

Then the probability for reaching optimal reward for $M_i$ is less than or equal to

\[(1 - \frac{1}{i})^{i-1}(\frac{1}{ki})\]

So the expected number of episodes to reach this optimal reward (and thus find an optimal policy) is
\[\frac{1}{(1 - \frac{1}{i})^{i-1}(\frac{1}{ki})} = (i-1)k(\frac{i}{i-1})^i \leq 4k(i-1)\]
since $i \geq 2$ and $(\frac{i}{i-1})^i$ is decreasing. By Chebyshev's inequality, a successful visit can be found in $4k(i-1)\log(H/\delta)$ with a probability at least $1-\delta/H$.

The expected total number of episodes for the all the MDP's is therefore
\[\sum_{j=2}^H 4k(j-1) \log(H/\delta) \leq \frac{H}{2} (4kH)\log(H/\delta)\]
which is $O(kH^2\log(H/\delta))$.
\end{proof}

\section{Generic Upper Bound for Sample Complexity}
\label{app:generic_bound}
In this section, we prove the generic upper bound on sample complexity in Theorem \ref{thm:generic_result}. We first prove Lemma \ref{lem:generic_rounds}, which holds under the same condition of Theorem \ref{thm:generic_result}.

\begin{lem}
\label{lem:generic_rounds}
Consider a multitask RL problem with MDP set $\mathcal{M}$ and value function class $\mathcal{F}$ such that $\mathcal{M}$ is $(\tilde{\alpha}, \tilde{c})$-diverse. Then Algorithm~\ref{alg:generic} running $T$ rounds with exploration function $\expl$ satisfies that with a probability at least $1-\delta$, the total number of rounds, where there exists an MDP $M$, such that $\pi^{\hat{f}_{t, M}}$ is $\beta$-suboptimal for $M$, can be upper bounded by
$$
     \mathcal{O}\left(|\mathcal{M}| H^2 d_{BE}(\gF, \Pi_{\gF}, 1/\sqrt{T})\frac{\ln \tilde{c}(\beta)}{\tilde{\alpha}({\beta})^2} \ln \left(\frac{{N}'_{\mathcal{F}}\left(T^{-1}\right) \ln T}{\delta}\right)\right),
$$
where $\operatorname{dim}_{BE}(\mathcal{F}, \Pi_{\gF}, 1/\sqrt{T})$ is the Bellman-Eluder dimension
of class $\mathcal{F}$ and $N'_{\mathcal{F}}(\rho) = \sum_{h = 1}^{H-1} N_{\mathcal{F}_h}(\rho) N_{\mathcal{F}_{h+1}}(\rho)$.
\end{lem}

\begin{proof}

Let us partition $\mathcal{F}_{\beta}$ into $\mathcal{F}_{\beta} = \{\mathcal{F}_{M, i}\}_{M \in \mathcal{M}, i \in [i_{\text{max}}]}$ such that
$$
    \mathcal{F}_{M, i} \coloneqq \{f \in \mathcal{F}_{\beta}: c(f, \mathcal{M}, \mathcal{F}) \in [e^{i-1}, e^i] \text{ and } M(f, \mathcal{M}, \mathcal{F}) = M\}.
$$
Furthermore, denote $(\hat f_{t, M})_{M \in \mathcal{M}}$ by $\hat f_t$. We define $\mathcal{K}_{M, i, t} = \{\tau \in [t], \hat f_{\tau} \in \mathcal{F}_{M, i}\}$. The proof in \cite{dann2022guarantees} can be seen as bounding the sum of $\mathcal{K}_{M, i, t}$ for a specific $M$, while apply the same bound for each $M$, which leads to an extra $|\mathcal{M}|$ factor.

\begin{lem}
Under the same condition in Theorem \ref{thm:generic_result} and the above definition, we have
$$
    |\mathcal{K}_{M, i, T}| \leq \mathcal{O}\left(\frac{H^2 d_{BE}(\gF, \Pi_{\gF}, 1/\sqrt{T})}{ \tilde{\alpha}({\beta})^2} \ln \frac{N_{\mathcal{F}}^{\prime}(1 / T) \ln (T)}{\delta}\right).
$$
\end{lem}
\begin{proof}
In the following proof, we fix an MDP $M$ and without further specification, the policies or rewards are with respect to the specific $M$. We study all the steps $t \in \mathcal{K}_{M, i, T}$.

For each $t \in \mathcal{K}_{M, i, T}$, 
\begin{enumerate}
    \item Recall that $\hat \pi_t$ is the mixture of exploration policy for all the MDPs: $\operatorname{Mixture}(\{\expl(\hat{\pi}_{t_,M'})\}_{M' \in \mathcal{M}})$;
    \item Define $\pi_t'$ as the improved policy that attains the maximum in the multitask myopic exploration gap for $\hat{f}_t$ in Definition \ref{defn:myopic_exploration_gap}.
\end{enumerate}
Note that $\pi_t'$ is a policy for $M$ since $t \in \mathcal{K}_{M, i, t}$. A key step in our proof is to upper bound the difference between the value of the current policy and the value of $\pi_t'$. By Lemma \ref{lem:value_difference}, The total difference in return between the greedy policies and the improved policies can be bounded by 
\begin{equation}
\label{equ:regret_decompose}
    \sum_{t \in \mathcal{K}_{M, i,T}}(V_{1, M}^{\pi_t^{\prime}}(s_1)-V_{1, M}^{\hat\pi_{t, M}}(s_1)) \leq \sum_{t \in \mathcal{K}_{M, i,T}} \sum_{h=1}^H \mathbb{E}^M_{\hat\pi_{t, M}}[(\mathcal{E}^M_h \hat{f}_{t, M})(s_h, a_h)]-\sum_{t \in \mathcal{K}_{M,i, T}} \sum_{h=1}^H \mathbb{E}^M_{\pi_t^{\prime}}[(\mathcal{E}^M_h \hat{f}_{t, M})(s_h, a_h)],
\end{equation}
where the exportation is taken over the randomness of the trajectory sampled for MDP $M$.

Under the completeness assumption in Assumption \ref{aspt:function_class}, by Lemma \ref{lem:self_norm} we show that with a probability $1-\delta$ for all $(h, t) \in [H] \times [T]$,
$$
    \sum_{\tau=1}^{t-1} \mathbb{E}^M_{\hat{\pi}_\tau}\left[\left(\mathcal{E}_h f_{t, M}\right)\left(s_h, a_h\right)\right]^2 \leq 3 \frac{t-1}{T}+176 \ln \frac{6 N_{\mathcal{F}}^{\prime}(1 / T) \ln (2 t)}{\delta}.
$$
We consider only the event, where this condition holds. Since $c(\hat f_t, \mathcal{M}, \mathcal{F}) \leq e^i$ for all $t \in \mathcal{K}_{M, i, T}$, by Definition \ref{defn:myopic_exploration_gap} we bound
\begin{align*}
    &\quad \sum_{\tau \in \mathcal{K}_{M, i, t-1}} \mathbb{E}^M_{\pi_\tau^{\prime}}\left[(\mathcal{E}_h^M \hat{f}_{t, M})\left(s_h, a_h\right)\right]^2  \\
    &\leq \sum_{\tau \in [t-1]} \mathbb{E}^M_{\pi_\tau^{\prime}}\left[(\mathcal{E}_h^M \hat{f}_{t, M})\left(s_h, a_h\right)\right]^2 \\
    &\leq e^i \sum_{\tau \in [t-1]} \mathbb{E}^M_{\hat \pi_\tau}\left[(\mathcal{E}_h^M \hat{f}_{t, M})\left(s_h, a_h\right)\right]^2 \\
    &\leq 179 e^i \ln \frac{6 N_{\mathcal{F}}^{\prime}(1 / T) \ln (2 t)}{\delta}.
\end{align*}

Combined with the distributional Eluder dimension machinery in Lemma \ref{lem:dist_eluder}, this implies that
\begin{align*}
    \sum_{t \in \mathcal{K}_{M, i, T}}\left|\mathbb{E}^M_{\pi_t^{\prime}}\left[\left(\mathcal{E}^M_h \hat{f}_{t, M}\right)\left(s_h, a_h\right)\right]\right| \leq \mathcal{O}\left(\sqrt{e^i d_{BE}(\gF, \Pi_{\gF}, 1/\sqrt{T}) \ln \frac{N_{\mathcal{F}}^{\prime}(1 / T) \ln (T)}{\delta}\left|\mathcal{K}_{M, i, T}\right|}\right.\\
    \left.+\min \left\{\left|\mathcal{K}_{M, i, T}\right|, d_{BE}(\gF, \Pi_{\gF}, 1/\sqrt{T})\right\}\right),
\end{align*}
Note that we can derive the same upper-bound for $\sum_{t \in \mathcal{K}_{M, i, T}}\left|\mathbb{E}^M_{\pi_t}\left[\left(\mathcal{E}^M_h \hat{f}_{t, M}\right)\left(s_h, a_h\right)\right]\right|$.
Then plugging the above two bounds into Equation (\ref{equ:regret_decompose}), we obtain
$$
    \sum_{t \in \mathcal{K}_{M, i,T}}(V_{1, M}^{\pi_t^{\prime}}(s_1)-V_{1, M}^{\hat{\pi}_{t, M}}(s_1)) \leq \mathcal{O}\left(\sqrt{e^iH^2 d_{BE}(\gF, \Pi_{\gF}, 1/\sqrt{T}) \ln \frac{N_{\mathcal{F}}^{\prime}(1 / T) \ln (T)}{\delta}\left|\mathcal{K}_{M, i, T}\right|}+Hd\left(\mathcal{F}_i^{\prime}\right)\right).
$$
By the definition of myopic exploration gap, we lower bound the LHS by
$$
    \sum_{t \in \mathcal{K}_{M, i,T}}(V_{1, M}^{\pi_t^{\prime}}(s_1)-V_{1, M}^{\hat{\pi}_{t, M}}(s_1)) \geq\left|\mathcal{K}_{M, i, T}\right| \sqrt{e^{i-1}} \alpha_{\beta}.
$$
Combining both bounds and rearranging yields
$$
     |\mathcal{K}_{M, i, T}| \leq \mathcal{O}\left(\frac{H^2 d_{BE}(\gF, \Pi_{\gF}, 1/\sqrt{T}) }{\alpha^2_{\beta}} \ln \frac{N_{\mathcal{F}}^{\prime}(1 / T) \ln (T)}{\delta}\right).
$$
\end{proof}

Summing over $M\in \mathcal{M}$ and $i \leq i_{max} < \ln \tilde{c}(\beta)$, we conclude Lemma \ref{lem:generic_rounds}.
\end{proof}

To convert Lemma \ref{lem:generic_rounds} into a sample complexity bound in Theorem \ref{thm:generic_result}, we show that for all $M$, $\hat{\pi}_{M} = \operatorname{Mixture}(\{\pi^{\hat{f}_{t, M}}\})$ is $\beta$-optimal for $M$.

$$
    \max_{\pi}V_{1, M}^{\pi} - V^{\hat{\pi}_{M}}_{1, M}(s_1) = \mathcal{O}\left(\frac{\beta|\mathcal{M}| H^2 d_{BE}(\gF, \Pi_{\gF}, 1/\sqrt{T})\frac{\ln \tilde{c}(\beta)}{\tilde{\alpha}({\beta})^2} \ln \left(\frac{{N}'_{\mathcal{F}}\left(T^{-1}\right) \ln T}{\delta}\right)}{T}\right).
$$
To have the above suboptimality controlled at the level $\beta$, we will need
$$
    \mathcal{O}\left(\frac{\beta|\mathcal{M}| H^2 d_{BE}(\gF, \Pi_{\gF}, 1/\sqrt{T})\frac{\ln \tilde{c}(\beta)}{\tilde{\alpha}({\beta})^2} \ln \left(\frac{{N}'_{\mathcal{F}}\left(T^{-1}\right) \ln T}{\delta}\right)}{T}\right) = \beta.
$$
Assume that $d_{BE}(\gF, \Pi_{\gF}, \rho) = \mathcal{O}(d_{\gF}^{\text{BE}}\log(1/\rho))$ and $\log(N'(\gF)(\rho)) = \mathcal{O}(d_{\gF}^{\text{cover}}\log(1/\rho))$ by ignoring other factors, which holds for most regular function classes including tabular and linear classes \citep{russo2013eluder}, we have 
$$
    T = \mathcal{O}\left(|\mathcal{M}| H^2 d^{BE}_{\gF}d^{\text{cover}}_{\gF}\frac{\ln \tilde{c}(\beta)}{\tilde{\alpha}({\beta})^2} \ln\frac{1}{\delta} \ln(1/\rho)\right),
$$
where $\rho^{-1} = \mathcal{O}\left({|\mathcal{M}| H^2 d^{\text{BE}}_{\gF}d^{\text{cover}}_{\gF}\frac{\ln \tilde{c}(\beta)}{\tilde{\alpha}({\beta})^2}}\ln(1/\delta)\right)$. The final bound has an extra $|\gM|$ dependence because we execute a policy for each MDP in a round.

% Theorem \ref{thm:generic_result} can be shown by taking the mixture policies $\operatorname{Mixture}(\{\pi^{f_{t, M}}\}_{t = 1}^T)$, which is 

\subsection{Technical Lemmas}

\begin{lem}[\textbf{Lemma 3} \citep{dann2022guarantees}]
\label{lem:value_difference}
For any MDP $M$, let $f = \{f_h\}_{h \in [H]}$ with $f_h: \mathcal{S} \times \mathcal{A} \mapsto \mathbb{R}$ and $\pi^f$ is the greedy policy of $f$. Then for any policy $\pi'$, 
$$
    V_1^{\pi^{\prime}}\left(s_1\right)-V_1^{\pi^f}\left(s_1\right) \leq \sum_{h=1}^H \mathbb{E}^M_{\pi^f}\left[\left(\mathcal{E}_h f\right)\left(s_h, a_h\right)\right]-\sum_{h=1}^H \mathbb{E}^M_{\pi^{\prime}}\left[\left(\mathcal{E}_h f\right)\left(s_h, a_h\right)\right].
$$
\end{lem}

\begin{lem}[Modified from \textbf{Lemma 4} \citep{dann2022guarantees}] Consider a sequence of policies $(\pi_t)_{t \in \mathbb{N}}$. At step $\tau$, we collect one episode using $\hat{\pi}_{\tau}$ and define $\hat f_{\tau}$ as the fitted Q-learning estimator up to step $t$ over the function class $\mathcal{F} = \{\mathcal{F}\}_{h \in [H]}$. Let $\rho \in \mathbb{R}^+$ and $\delta \in (0, 1)$. If $\mathcal{F}$ satisfies Assumption \ref{aspt:function_class}, then with a probability at least $1-\delta$, for all $h \in [H]$ and $t \in \mathbb{N}$,
$$
    \sum_{\tau = 1}^{t-1} \mathbb{E}^M_{\hat \pi_\tau} [(\mathcal{E}_h \hat f_{t})(s_h, a_h)]^2 \leq 3\rho t+176 \ln \frac{6 N_{\mathcal{F}}^{\prime}(\rho) \ln (2 t)}{\delta},
$$
where $N_{\mathcal{F}}^{\prime}(\rho)=\sum_{h=1}^H N_{\mathcal{F}_h}(\rho) N_{\mathcal{F}_{h+1}}(\rho)$ is the sum of $\ell_{\infty}$ covering number of $\mathcal{F}_{h} \times \mathcal{F}_{h+1}$ w.r.t. radius $\rho > 0$.
\label{lem:self_norm}
\end{lem}
\begin{proof}
    The only difference between our statement and the statement in \cite{dann2022guarantees} is that they consider $\hat{\pi}_{\tau} = \expl(\hat{f_{\tau}})$, while this statement holds for any data-collecting policy $\hat{\pi}_{\tau}$. To show this, we go through the complete proof here.
    
    Consider a fixed $t \in \mathbb{N}$, $h\in[H]$ and $f = \{f_h, f_{h+1}\}$ with $f_h \in \mathcal{F}_h$, $f_{h+1} \in \mathcal{F}_{h+1}$. Let $(x_{t, h}, a_{t, h}, r_{t, h})_{t \in \mathbb{N}, h \in [H]}$ be the collected trajectory in $[t]$. Then
    $$
        \begin{aligned} Y_{t, h}(f)= & \left(f_h\left(x_{t, h}, a_{t, h}\right)-r_{t, h}-\max _{a^{\prime}} f_{h+1}\left(x_{t, h+1}, a^{\prime}\right)\right)^2-\left(\left(\mathcal{T}_h f_{h+1}\right)\left(x_{t, h}, a_{t, h}\right)-r_{t, h}-\max _{a^{\prime}} f_{h+1}\left(x_{t, h+1}, a^{\prime}\right)\right)^2 \\ = & \left(f_h\left(x_{t, h}, a_{t, h}\right)-\left(\mathcal{T}_h f_{h+1}\right)\left(x_{t, h}, a_{t, h}\right)\right) \\ & \times\left(f_h\left(x_{t, h}, a_{t, h}\right)+\left(\mathcal{T}_h f_{h+1}\right)\left(x_{t, h}, a_{t, h}\right)-2 r_{t, h}-2 \max _{a^{\prime}} f_{h+1}\left(x_{t, h+1}, a^{\prime}\right)\right).\end{aligned}
    $$
    Let $\mathfrak{F}_t$ be the $\sigma$-algebra under which all the random variables in the first $t-1$ episodes are measurable. Note that $|Y_{t, h}(f)| \leq 4$ almost surely and the conditional expectation of $Y_{y, h}(f)$ can be written as 
    $$
        \mathbb{E}\left[Y_{t, h}(f) \mid \mathfrak{F}_t\right]=\mathbb{E}\left[\mathbb{E}\left[Y_{t, h}(f) \mid \mathfrak{F}_t, x_{t, h}, a_{t, h}\right] \mid \mathfrak{F}_t\right]=\mathbb{E}_{\pi_{t}}[\left(f_h-\mathcal{T}_h f_{h+1}\right)\left(x_h, a_h\right)^2].
    $$
    The variance can be bounded by 
    $$
        \operatorname{Var}\left[Y_{t, h}(f) \mid \mathfrak{F}_t\right] \leq \mathbb{E}\left[Y_{t, h}(f)^2 \mid \mathfrak{F}_t\right] \leq 16 \mathbb{E}\left[\left(f_h-\mathcal{T}_h f_{h+1}\right)\left(x_{t, h}, a_{t, h}\right)^2 \mid \mathfrak{F}_t\right]=16 \mathbb{E}\left[Y_{t, h}(f) \mid \mathfrak{F}_t\right],
    $$
    where we used the fact that $|f_h(x_{t, h}, a_{t, h})+(\mathcal{T}_h f_{h+1})(x_{t, h}, a_{t, h})-2 r_{t, h}-2 \max _{a^{\prime}} f_{h+1}(x_{h+1}, a^{\prime})| \leq 4$ almost surely. Applying Lemma \ref{lem:time-unif-ineq} to the random variable $Y_{t, h}(f)$, we have that with probability at least $1-\delta$, for all $t \in \mathbb{N}$, 
    $$
        \begin{aligned}  \sum_{i=1}^t \mathbb{E}\left[Y_{i, h}(f) \mid \mathfrak{F}_i\right] &\leq 2 A_t \sqrt{\sum_{i=1}^t \operatorname{Var}\left[Y_{i, h}(f) \mid \mathfrak{F}_i\right]}+12 A_t^2+\sum_{i=1}^t Y_{i, h}(f) \\ & \leq 8 A_t \sqrt{\sum_{i=1}^t \mathbb{E}\left[Y_{i, h}(f) \mid \mathfrak{F}_i\right]}+12 A_t^2+\sum_{i=1}^t Y_{i, h}(f), \end{aligned}
    $$
    where $A_t = \sqrt{2\ln\ln(2t) + \ln(6/\delta)}$. Using AM-GM inequality and rearranging terms in the above we have
    $$
        \sum_{i=1}^t \mathbb{E}\left[Y_{i, h}(f) \mid \mathfrak{F}_i\right] \leq 2 \sum_{i=1}^t Y_{i, h}(f)+88 A_t^2 \leq 2 \sum_{i=1}^t Y_{i, h}(f)+176 \ln \frac{6 \ln (2 t)}{\delta}.
    $$
    Let $\mathcal{Z}_{\rho, h}$ be a $\rho$-cover of $\mathcal{F}_h \times \mathcal{F}_{h+1}$. Now taking a union bound over all $\phi_h \in \mathcal{Z}_{\rho, h}$ and $h \in [H]$, we obtain that with probability at least $1-\delta$ for all $\phi_h$ and $h \in [H]$
    $$
        \sum_{i=1}^t \mathbb{E}\left[Y_{i, h}\left(\phi_h\right) \mid \mathfrak{F}_i\right] \leq 2 \sum_{i=1}^t Y_{i, h}\left(\phi_h\right)+176 \ln \frac{6 N_{\mathcal{F}}^{\prime}(\rho) \ln (2 t)}{\delta}.
    $$
    This implies that with probability at least $1-\delta$ for all $f = \{f_h, f_{h+1}\} \in \mathcal{F}_h \times \mathcal{F}_{h+1}$ and $h \in [H]$, 
    $$
        \sum_{i=1}^t \mathbb{E}\left[Y_{i, h}(f) \mid \mathfrak{F}_i\right] \leq 2 \sum_{i=1}^t Y_{i, h}(f)+3 \rho(t-1)+176 \ln \frac{6 N'_{\mathcal{F}}(\rho) \ln (2 t)}{\delta}.
    $$
    Let $\hat f_{t, h}$ be the $h$-th component of the function $\hat f_t$. The above inequality holds in particular for $f = \{\hat f_{t, h}, \hat f_{t, h+1}\}$ for all $t \in \mathbb{N}$. Finally, we have 
    $$\begin{aligned}  \sum_{i=1}^{t-1} Y_{i, h}\left(\widehat{f}_t\right)
    &= \sum_{i=1}^{t-1}\left(\widehat{f}_{t, h}\left(s_{i, h}, a_{i, h}\right)-r_{i, h}-\max _{a^{\prime}} \widehat{f}_{t, h+1}\left(s_{i, h+1}, a^{\prime}\right)\right)^2 \\ &\quad -\sum_{i=1}^{t-1}\left(\left(\mathcal{T}_h \widehat{f}_{t, h+1}\right)\left(s_{i, h}, a_{i, h}\right)-r_{i, h}-\max _{a^{\prime}} \widehat{f}_{t, h+1}\left(s_{i, h+1}, a^{\prime}\right)\right)^2 \\ &= \inf _{f^{\prime} \in \mathcal{F}_h} \sum_{i=1}^{t-1}\left(f^{\prime}\left(s_{i, h}, a_{i, h}\right)-r_{i, h}-\max _{a^{\prime}} \widehat{f}_{t, h+1}\left(s_{i, h+1}, a^{\prime}\right)\right)^2 \\ &\quad -\sum_{i=1}^{t-1}\left(\left(\mathcal{T}_h \widehat{f}_{t, h+1}\right)\left(s_{i, h}, a_{i, h}\right)-r_{i, h}-\max _{a^{\prime}} \widehat{f}_{t, h+1}\left(s_{i, h+1}, a^{\prime}\right)\right)^2 \\ & \leq 0,\end{aligned}$$
where the last inequality follows from the completeness in Assumption \ref{aspt:function_class}.
\end{proof}

\begin{lem}[Time-Uniform Freedman Inequality]
\label{lem:time-unif-ineq}
Suppose $\{X_{t}\}_{t = 1}^{\infty}$ is a martingale difference sequence with $|X_t| \leq b$. Let 
$$
    \operatorname{Var}_{\ell}\left(X_{\ell}\right)=\operatorname{Var}\left(X_{\ell} \mid X_1, \cdots, X_{\ell-1}\right).
$$
Let $V_t = \sum_{\ell = 1}^t \operatorname{Var}_{\ell}(X_\ell)$ be the sum of conditional variances of $X_t$. Then we have that for any $\delta' \in (0, 1)$ and $t \in \mathbb{N}$
$$
    \mathbb{P}\left(\sum_{\ell=1}^t X_{\ell}>2 \sqrt{V_t} A_t+3 b A_t^2\right) \leq \delta^{\prime}
$$
where $A_t  = \sqrt{2\ln\ln(2(\max(V_t/b^2, 1))) + \ln(6/\delta')}$.
\end{lem}

\begin{lem}[Lemma 41 \citep{jin2021bellman}]
\label{lem:dist_eluder}
Given a function class $\Phi$ defined on $\mathcal{X}$ with $|\phi(x)| \leq C$ for all $(\phi, x) \in \Phi \times \mathcal{X}$ and a family of probability measures $\Pi$ over $\mathcal{X}$. Suppose sequences $\{\phi_i\}_{i \in [K]} \subset \Phi$ and $\{\mu_{i}\}_{i \in [K]} \subset \Pi$ satisfy for all $k \in [K]$ that $\sum_{i = 1}^{k-1} (\E_{\mu_i}[\phi_k])^2 \leq \beta$. Then for all $k \in [K]$ and $w > 0$,
$$
    \sum_{t=1}^k\left|\mathbb{E}_{\mu_t}\left[\phi_t\right]\right| \leq O\left(\sqrt{\operatorname{dim}_{D E}(\Phi, \Pi, \omega) \beta k}+\min \left\{k, \operatorname{dim}_{D E}(\Phi, \Pi, \omega)\right\} C+k \omega\right).
$$
\end{lem}

\paragraph{Proof of Theorem \ref{thm:generic_result}.} Denote $B = {\Theta}\left(|\mathcal{M}|^2 H^2 d_{\operatorname{BE}}\frac{\ln \tilde{c}({\beta})}{\tilde{\alpha}({\beta})^2 \beta} \ln \left(\frac{\bar{N}_{\mathcal{F}}\left(T^{-1}\right) \ln T}{\delta}\right)\right)$. The following Corollary transform Lemma~\ref{lem:generic_rounds} to Theorem \ref{thm:generic_result}, whose proof directly follows by taking $T = B/\beta$. Since at most $B$ rounds are suboptimal according to Lemma~\ref{lem:generic_rounds}, the mixing of all $T$ policies are $\beta$-optimal. This leads to a sample complexity
$$
\mathcal{C}(\tilde{\alpha}, \tilde{c}) = {\Theta}\left(|\mathcal{M}|^2 H^2 d_{\operatorname{BE}}\frac{\ln \tilde{c}({\beta})}{\tilde{\alpha}({\beta})^2 \beta} \ln \left(\frac{\bar{N}_{\mathcal{F}}\left(T^{-1}\right) \ln T}{\delta}\right)\right)
$$
% \begin{cor}[Sample Complexity Guarantee]
% \label{cor:optimality_gurantee}
% Under the same condition in Lemma~\ref{lem:generic_rounds}, we the sample complexity of Algorithm~\ref{alg:generic} can be upper bounded by 
% $$
%     \mathcal{C}(\tilde{\alpha}, \tilde{c}) = {\Theta}\left(|\mathcal{M}|^2 H^2 d_{\operatorname{BE}}\frac{\ln \tilde{c}({\beta})}{\tilde{\alpha}({\beta})^2 \beta} \ln \left(\frac{\bar{N}_{\mathcal{F}}\left(T^{-1}\right) \ln T}{\delta}\right)\right).
% $$
% \end{cor}

\section{Omitted Proofs for Case Studies}
\label{app:case_studies}
\subsection{Linear MDP case}

Note that in this section, we use $\mathbb{E}_{\pi}$ for the expectation over transition w.r.t a policy $\pi$.

\begin{lem}
\label{lem:linear_coverage}
Let $\mathcal{F}$ be the function class in Proposition \ref{prop:linear_value_function}. For any policy $\pi$ such that 
$
    \lambda_{\text{min}}(\Phi_{h}^\pi) \geq \barbelow{\lambda},
$
then for any policy $\pi'$ and $f' \in \mathcal{F}$,
$
    \mathbb{E}_{\pi'}\left[\left(\mathcal{E}_h^2 f^{\prime}\right)\left(s_h, a_h\right)\right]  \leq \mathbb{E}_{\pi}\left[\left(\mathcal{E}_h^2 f^{\prime}\right)\left(s_h, a_h\right)\right]/\barbelow{\lambda}.
$
\end{lem}
\begin{proof}
    Recall that $\Phi_h^{\pi} \coloneqq \mathbb{E}_{\pi} \phi_h(s_h, a_h) \phi_h(s_h, a_h)^\top$.

    We derive the Bellman error term using the fact that $f'$ is a linear function and the transitions admit the linear function as well. For any policy $\pi$, we have
    \begin{align*}
        &\quad \mathbb{E}_{\pi}[(\mathcal{E}_h^2 f')(s_h, a_h)]\\ &= \mathbb{E}_{\pi}\left[\left(f'_h(s_h, a_h) - \phi_h(s_h, a_h)^\top \theta_h - \max_{a'}\mathbb{E}_{s_{h+1}}[f'_{h+1}(s_{h+1}, a') \mid s_h, a_h]\right)^2 \right] \\
        &= \mathbb{E}_{\pi}\left[\left(\phi_h(s_h, a_h)^\top w_h - \phi_h(s_h, a_h)^\top \theta_h - \max_{a'}\mathbb{E}_{s_{h+1}}[\phi_{h+1}(s_{h+1}, a')^\top w_{h+1} \mid s_h, a_h]\right)^2 \right] \\
        &= \mathbb{E}_{\pi}\left[\left(\phi_h(s_h, a_h)^\top w_h - \phi_h(s_h, a_h)^\top \theta_h - \phi_h(s_h, a_h)^\top\int_{s'}\phi_{h+1}(s', \pi^{f'}_{h+1}(s'))^\top w_{h+1} \mu_{h}(s') ds'\right)^2 \right]\\
        &= \mathbb{E}_{\pi}\left[\left(\phi_h(s_h, a_h)^\top (w_h - \theta_h - w'_{h+1})\right)^2 \right] \\
        &= (w_h - \theta_h - w'_{h+1})^\top\mathbb{E}_{\pi}\left[\phi_h(s_h, a_h) \phi_h(s_h, a_h)^\top \right] (w_h - \theta_h - w'_{h+1})
    \end{align*}
    where $w_{h+1}' = \int_{s'}\phi_{h+1}(s', \pi^{f'}_{h+1}(s'))^\top w_{h+1} \mu_{h}(s') ds'$. Since by the assumption in Definition~\ref{defn:linear_MDP} that 
    $\|\phi_h(s, a)\| \leq 1$ for any $s, a$, we have $\Phi_{h}^{\pi'} \prec I$. The result follow by the condition that $\lambda_{\text{min}}(\Phi_h^{\pi}) \geq \underline{\lambda}$.
\end{proof}

{\bf Lemma \ref{lem:linear_2}.}{\it
Fix a step $h$. Let $\{M_{i, h}\}_{i \in [d]}$ be the $d$ MDPs such that $ \theta_{h, M_{i, h}} = e_i$ as in Definition~\ref{defn:linear_MDP}. Let $\{\pi_i\}_{i=1}^d$ be $d$ policies such that $\pi_i$ is a $\beta$-optimal policy for $M_{i, h}$ with $\beta < b_1 /2$. Let $\tilde{\pi} = \operatorname{Mixture}(\{\expl(\pi_i)\}_{i = 1}^d)$. Then for any $\nu \in \mathbb{S}^{d-1}$, we have
$
    \lambda_{\operatorname{min}}(\Phi_{h+1}^{\tilde{\pi}})\geq {\epsilon_h\prod_{h'=1}^{h-1}(1-\epsilon_{h'}) b_1^2}/{(2dA)}.
$
}
\begin{proof}
Let $\pi$ be any stationary policy and recall that $\Pi$ is the set of all the stationary policies. We denote $A_h^{\pi}(s') \sim \pi_h(s')$ by the random variable for the action sampled at the step $h$ using policy $\pi$ given the state is $s'$. Let $\phi_h^{\pi} \coloneqq \E_{\pi} \phi_h(s_h, a_h)$.

We further define 
$$
    a_{h+1}^{\nu}(s) \coloneqq \argmax_{a \in \mathcal{A}} [\nu^{\top} \phi_{h+1}(s, a) \phi_{h+1}(s, a)^{\top}\nu].
$$

Lower bound the following quadratic term for any unit vector $\nu \in \mathbb{R}^d$,
\begin{align*}
    &\quad \max_{\pi \in \Pi} \nu^\top \Phi_{h+1}^{\pi}\nu \\
    % &\quad\quad\text{(By the definition of linear MDP)}\\
    &= \max_{\pi \in \Pi} \mathbb{E}_{\pi}\left[ \int_{s'}\nu^\top\phi_{h+1}(s', A_{h+1}^{\pi}(s'))\phi_{h+1}(s', A_{h+1}^{\pi}(s'))^\top\nu \mu_{h}(s')^\top \phi_{h}(s_h, a_h) ds'\right]\\
    &= \max_{\pi} \mathbb{E}_{\pi}[\phi_{h}(s_h, a_h)^\top] \left(\int_{s'}\nu^\top\phi_{h+1}(s', a_{h+1}^\nu(s'))\phi_{h+1}(s', a_{h+1}^\nu(s'))^\top\nu \mu_{h}(s') ds'\right) \\
    &= \max_{\pi \in \Pi} (\phi_h^{\pi})^\top w_{h+1}^{\nu}.
\end{align*} 
where %the second equality holds because $\mu_h(s') > 0$ and 
we let $w_{h+1}^{\nu} \coloneqq  \int_{s'}\nu^\top\phi_{h+1}(s', a_{h+1}^\nu(s'))\phi_{h+1}(s', a_{h+1}^\nu(s'))^\top\nu \mu_{h}(s')^\top ds'.$

By Assumption \ref{aspt:linear_cover}, we have $\max_{\pi \in \Pi} \mathbb{E}_{\pi}[\phi_{h}(s_h, a_h)^\top] w_{h+1}^{\nu} \geq b_1^2$.

For the mixture policy $\tilde{\pi}$ defined in our lemma, 
\begin{align*}
    \nu^\top \Phi_{h+1}^{\tilde{\pi}} \nu
    &= \frac{1}{d}\sum_{i = 1}^d\mathbb{E}_{\expl(\pi_i)}[\nu^\top\phi_{h+1}(s_{h+1}, a_{h+1})\phi_{h+1}(s_{h+1}, a_{h+1})^\top\nu] \\
    & \geq \frac{\epsilon_h\prod_{h'=1}^{h-1}(1-\epsilon_{h'})}{Ad} \sum_{i = 1}^d (\phi_h^{\pi_i})^\top w_{h+1}^{\nu}. \numberthis \label{equ:linear_decomposing}
\end{align*}

Since $\pi_i$ is a $b_1/2$-optimal policy for MDP $M_{i, h}$ and again by Assumption \ref{aspt:linear_cover}, we have 
\begin{equation}
    \theta_{h, M_{i, h}}^\top \phi_{h}^{\pi_i} \geq \frac{1}{2}\max_{\pi \in \Pi} \theta_{h, M_{i, h}}^\top \phi_h^{\pi}. \label{equ:linear_induction}
\end{equation}
For any vector $\nu \in \mathbb{R}^d$, let $[\nu]_i$ be the $i$-th dimension of the vector. Note that $\theta_{h, M_{i, h}} = e_i$, (\ref{equ:linear_induction}) indicates $[\phi_{h}^{\pi_i}]_i \geq \frac{1}{2} \max_{\pi}[\phi_{h}^{\pi}]_i$.

\iffalse
For any unit vector $w$,
\begin{align*}
    & \quad\sum_{i = 1}^d (\phi_h^{\pi_i})^{\top} ww^{\top} \phi_{h}^{\pi_i} \\
    &\geq \sum_{i = 1}^d (\phi_h^{\pi_i})^{\top} \theta_{h, M_{i, h}}\theta_{h, M_{i, h}}^{\top} \phi_{h}^{\pi_i} \\
    &\geq \frac{1}{4d}\sum_{i = 1}^d \max_{\pi} (\phi_h^{\pi})^{\top} \theta_{h, M_{i, h}}\theta_{h, M_{i, h}}^{\top} \phi_h^{\pi} \\
    &\geq \frac{1}{4d}\max_{\pi} (\phi_h^{\pi})^{\top} \sum_{i = 1}^d (\theta_{h, M_{i, h}}\theta_{h, M_{i, h}}^{\top}) \phi_h^{\pi} \\
    &\geq \frac{1}{4d\gamma^2} \max_{\pi} (\phi_{h}^{\pi})^T ww^{\top} \phi_{h}^{\pi},
\end{align*}
where the first inequality is by the fact that $\|\theta_{h, M_{i, h}}\|\leq \sqrt{d}$.
\fi
Combining the inequality (\ref{equ:linear_induction}) with (\ref{equ:linear_decomposing}), we have
\begin{align*}
    \nu^\top\Phi_{h+1}^{\tilde{\pi}}\nu &= \frac{\epsilon_h\prod_{h'=1}^{h-1}(1-\epsilon_{h'})}{dA} \sum_{i = 1}^d \sum_{j = 1}^d [\phi_{h}^{\pi_i}]_j [w_{h+1}^{\nu}]_j\\
    &\geq \frac{\epsilon_h\prod_{h'=1}^{h-1}(1-\epsilon_{h'})}{dA} \sum_{i = 1}^d  [\phi_{h}^{\pi_i}]_i [w_{h+1}^{\nu}]_i\\
    &\geq \frac{\epsilon_h\prod_{h'=1}^{h-1}(1-\epsilon_{h'})}{dA} \sum_{i = 1}^d \max_{\pi} [\phi_{h}^{\pi}]_i [w_{h+1}^{\nu}]_i\\
    &\geq \frac{\epsilon_h\prod_{h'=1}^{h-1}(1-\epsilon_{h'})}{2dA} \max_{\pi} (\phi_{h}^{\pi})^{\top} w_{h+1}^{\nu}\\
    &\geq \frac{\epsilon_h\prod_{h'=1}^{h-1}(1-\epsilon_{h'}) b_1^2}{2dA}
    % &= \frac{\epsilon_h(1-\epsilon_h)^h}{2dA} \max_{\pi}\nu^\top \Phi_{h+1}^{\pi} \nu 
\end{align*}
\end{proof}

\subsection{Proof of Theorem \ref{thm:linear}}
{\bf Theorem \ref{thm:linear}.} {\it Consider $\mathcal{M}$ defined in Definition \ref{def:linear_diverse}. With Assumption \ref{aspt:linear_cover} holding and $\beta \leq b_1/2$, for any $f \in \mathcal{F}_{\beta}$, we have lower bound
$
    \alpha(f, \mathcal{F}, \mathcal{M}) \geq \sqrt{{e\beta^2b_1^2}/(2A|\mathcal{M}|H)}
$ by setting $\epsilon_h = 1/h$.}

\begin{proof}
Let $h'$ be the smallest $h$, such that there exists $M_{i, h}$, $\pi^{f_{M_{i, h}}}$ is $\beta$-suboptimal. Let $(i', h')$ be the index of the MDP that has the suboptimal policy. We show that $M_{i', h'}$ has lower bounded myopic exploration gap.

By definition, $f$ is $\beta$-optimal for any MDP $M_{i, h'-1}$. By Lemma \ref{lem:linear_2}, letting $\tilde{\pi} = \expl(f, \epsilon_{h'})$, we have
$$
    \nu^\top \Phi_{h'+1}^{\tilde{\pi}} \nu \geq \frac{\epsilon_{h'}\prod_{h''=1}^{h'-1}(1-\epsilon_{h''})b_1^2}{2A|\mathcal{M}|}. 
    % \footnote{\text{Note that we have $H$ steps, which introduced an extra factor of $H$}.}
$$

By Lemma \ref{lem:linear_coverage}, we have that the optimal value function $f^*$ for MDP $M_{i', h'}$ satisfies that for any $f'$
$$
    \mathbb{E}_{\pi^{f^*}}^M\left[\left(\mathcal{E}_h^2 f^{\prime}\right)\left(s_h, a_h\right)\right] \leq \frac{2A|\mathcal{M}|}{\epsilon_{h'}\prod_{h''=1}^{h'-1}(1-\epsilon_{h''})b_1^2} \mathbb{E}_\pi^M\left[\left(\mathcal{E}_h^2 f^{\prime}\right)\left(s_h, a_h\right)\right].
$$
Thus, by Definition \ref{defn:myopic_exploration_gap}, the myopic exploration gap for $f$ is lower bounded by
$$
    \beta \frac{1}{\sqrt{c}} = \beta \sqrt{\frac{\epsilon_{h'}\prod_{h''=1}^{h'-1}(1-\epsilon_{h''})b_1^2}{2A|\mathcal{M}|}} \geq \sqrt{\frac{\beta^2b_1^2}{2A|\mathcal{M}|eH}},
$$
if we choose $\epsilon_h = 1/(h+1)$.
\end{proof}

\subsection{Linear Quadratic Regulator}
To demonstrate the generalizability of the proposed framework, we introduce another interesting setting called Linear Quadratic Regulator (LQR). LQR takes continuous state space $\mathbb{R}^{d_s}$ and action space $\mathbb{R}^{d_a}$. In the LQR system, the state $s_{h} \in \mathbb{R}^{d_s}$ evolves according to the following transition:
$
    s_{h+1} = A_h s_{h} + B_h a_{h},% + \eta_h,
$
%where $\eta_h \sim \mathcal{N}(0, \Sigma)$ is the system noise with \textit{known} Gaussian distribution, 
where $A_h \in \mathbb{R}^{d_s \times d_s}$, $B_h \in \mathbb{R}^{d_s \times d_a}$ are unknown system matrices that are shared by all the MDPs. We denote $s_h = (s_h, a_h)$ as the state-action vector. The reward function for an MDP $M$ takes a known quadratic form $r_{h, M}(s, a) = s^\top R_{h, M}^s s + a^\top R_{h, M}^a a$, where $R_{h, M}^s \in \mathbb{R}^{d_s \times d_s}$ and $R_{h, M}^a \in \mathbb{R}^{d_a \times d_a}$ \footnote{Note that LQR system often consider a cost function and the goal of the agent is to minimize the cumulative cost with $R^s_{h, M}$ being semi-positive definite. We formulation this as a reward maximization problem for consistency. Thus, we consider $R^s_{h, M} \prec \bm{0}$}. 
%For simplicity, we let $s_0 = \bm{0}$.

Note that LQR is more commonly studied for the infinite-horizon setting, where stabilizing the system is a primary concern of the problem. We consider the finite-horizon setting, which alleviates the difficulties on stabilization so that we can focus our discussion on exploration. Finite-horizon LQR also allows us to remain consistent notations with the rest of the paper. A related work \citep{simchowitz2020naive} states that naive exploration is optimal for online LQR with a condition that the system injects a random noise onto the state observation with a full rank covariance matrix $\Sigma \succ 0$. Though this is a common assumption in LQR literature, one may notice that the analog of this assumption in the tabular MDP is that any state and action pair has a non-zero probability of visiting any other state, which makes naive exploration sample-efficient trivially. In this section, we consider a deterministic system, where naive exploration does not perform well in general.

\paragraph{Properties of LQR.} It can be shown that the optimal actions are linear transformations of the current state \citep{farjadnasab2022model,li2022model}. %Thus, a policy can be parametrized by $a_h = K_{h, M} s_h$ with $K_{h, M} \in \mathbb{R}^{d_a \times d_s}$. 

The optimal linear response is characterized by the discrete-time Riccati equation given by
$$
    P_{h, M} = A_h^\top (P_{h+1, M} - P_{h+1, M} \bar R_{h+1, M}^{-1} B_{h}^\top P_{h+1, M})A_h + R^s_{h, M},
$$
where $\bar R_{h+1, M} = R^a_{h} + B_h^\top P_{h+1, M} A_h$ and $P_{H+1} = \bm{0}$. Assume that the solution for the above equation is $\{P^*_{h, M}\}_{h \in [H+1]}$, then the optimal control actions takes the form 
$$
    a_h = F_{h, M}^* s_h, \text{ where } F_{h, M}^* = -(R^s_{h, M} + B_h^\top P_{h, M}^* B_h)^{-1} B^{\top}P_{h, M}^* A_h.
$$
and optimal value function takes the quadratic form: 
$
    V^*_{h, M}(s) = s^\top P^*_{h, M} s %+ \sum_{h' = h}^H\mathbb{E}[\eta_{h}^\top P^*_{h'+1, M} \eta_{h'}], \text{ and }
$ and 
$$
Q_{h, M}^*(x) = x^\top \left[\begin{array}{cc}
    R_{h, M}^s + A_h^\top P^*_{h+1, M} A_h & A_h^\top P^*_{h+1, M} B_{h} \\
      B_{h}^\top P^*_{h+1, M} A_h & R_{h, M}^a + B_h^\top P^*_{h+1, M} B_h
\end{array}\right] x. %+ \sum_{h' = h}^H\mathbb{E}[\eta_{h}^\top P^*_{h'+1, M} \eta_{h'}].
$$
This observation allows us to consider the following function approximation
$$
    \mathcal{F} = (\mathcal{F}_h)_{h \in [H+1]}, \text{ where each } \mathcal{F}_h = \{x \mapsto x^\top G_h x: G_h \in \mathbb{R}^{(d_s + d_a) \times (d_s + d_a)}\}.
$$
The quadratic function class satisfies Bellman realiazability and completeness assumptions.

\begin{defn}[Diverse LQR Task Set]
\label{defn:LQR_task}
Inspired by the task construction in linear MDP case, we construct the diverse LQR set by $\mathcal{M} = \{M_{i, h}\}_{i \in[d_s], h \in [H]}$ such that these MDPs all share the same transition matrices $A_h$ and $B_h$ and each $M_{i, h}$ has $R^s_{h', M_{i, h}} = \mathbbm{1}[h' = h] e_i e_i^\top$ and $R^a_{h', M_{i, h}} = -I$.
\end{defn}

\begin{aspt}[Regularity parameters]
\label{aspt:LQR_regularity}
Given the task set in Definition \ref{defn:LQR_task}, we define some constants that appears on our bound. Let $\pi^*_{i, h}$ be the optimal policy for $M_{i, h}$. Let 
$$
    b_4 = \max_{i, h} \mathbb{E}_{\pi^*_{i, h}} \max_{h'}\|s_{h'}\|_2, \text{ and } b_5 = \max_{i, h} \mathbb{E}_{\pi^*_{i, h}} \max_{h'}\|a_{h'}\|_2.
$$
\end{aspt}

These regularity assumption is reasonable because the optimal actions are linear transformations of states and we consider a finite-horizon MDP, with $F_h^*$ having upper bounded eigenvalues.

Similarly to the linear MDP case, we assume that the system satisfies some visibility assumption.
\begin{aspt}[Coverage Assumption]
\label{aspt:LQR_cover}
For any $\nu \in \mathbb{R}^{d_s - 1}$, there exists a policy $\pi$ with $\|a_h\|_2 \leq 1$ such that
$$
    \max_{\pi}\mathbb{E}_{\pi}[s_{h}^\top \nu ] \geq b_3, \text{ for } b_3 > 1.
$$
\end{aspt}
% \ziping{fix $s$ and $x$ inconsistency.}

\iffalse
Now assuming that we have a set of policies $\{\pi_i\}_{i \in [d_s]}$. The $i$-th policy is a $b_3b_5/2$-optimal policy for LQR with $R_{h, i}^s = -e_ie_i^T$ and $R_{h, i}^a = -I$. Let $\tilde{\pi} = \operatorname{Mixture}(\{\pi_i\}_{i \in [d_s]})$. Then we have
$$
    \lambda_{\operatorname{min}}(\mathbb{E}_{\tilde{\pi}} x_{h+1} x_{h+1}^T) > \min\{\frac{b_3^2b_5^2}{2d_s}, \sigma^2\}.
$$
\fi
% \ziping{We have a $b_5$ here because $\pi$ is allowed to go $b_5$.}

\begin{thm}
\label{thm:LQR}
Given Assumption \ref{aspt:LQR_regularity}, \ref{aspt:LQR_cover} and the diverse LQR task set in Definition \ref{defn:LQR_task}, we have that for any $f \in \mathcal{F}_{\beta}$ with $\beta \leq (b_3^2-1)b_5^2/2$, 
$$
    \alpha(f, \mathcal{F}, \mathcal{M}) = \Omega\left( \frac{\max\{b_4^2, b_5^2\}b_4^2}{d_sH\max\{(b_3^2-1)b_5^2, d_s\sigma^2\}(b_3^2-1)b_5^2} \right).
$$
\end{thm}

\subsection{Proof of Theorem \ref{thm:LQR}}
\begin{lem}
\label{lem:LQR_induction}
Assume that we have a set of policies $\{\pi_{i}\}_{i\in[d]}$ such that the $i$-th policy is a $(b_3^2-1)b_5^2/2$-optimal policy for LQR with $R^s_{h, i} = e_ie_i^\top$ and $R^a_{h, i} = -I$. Let $\tilde{\pi} = \text{Mixture}(\expl\{\pi_i\})$. Then we have 
$$
    \lambda_{\text{min}}(\mathbb{E}_{\tilde{\pi}} s_{h+1}s_{h+1}^\top) \geq \frac{d_s\max\{\underline{\lambda}, d\sigma^2\}}{2\max\{b_4^2, b_5^2\}\prod_{h'=1}^{h-1}(1-\epsilon_{h'})\epsilon_{h}} \underline{\lambda}, 
$$
with $\underline{\lambda} = (b_3^2-1)b_5^2$.
\end{lem}

\begin{proof}
We directly analyze the state covariance matrix at the step $h+1$. Let $\eta_h \sim \mathcal{N}(0, \sigma^2)$
\begin{align*}
    \E_{\tilde{\pi}}s_{h+1}s_{h+1}^\top 
    &= \E_{\tilde{\pi}}(A_h s_{h} + B_h a_h)(A_h s_{h} + B_h a_h)^\top \\
    &\succeq \frac{\prod_{h'=1}^{h-1}(1-\epsilon_{h'})\epsilon_{h}}{d_s} \sum_{i = 1}^{d_s} \left(\E_{\pi_i}(A_h s_h + B_h \eta_h)(A_h s_h + B_h \eta_h)^{\top} \right) \\
    &= \frac{\prod_{h'=1}^{h-1}(1-\epsilon_{h'})\epsilon_{h}}{d_s} \sum_{i = 1}^{d_s} \left(A_h \E_{\pi_i} s_hs_h^{\top} A_h^{\top} + B_h \E\eta_h\eta_h^\top B_h^\top \right) \numberthis \label{equ:cov_decomp}
\end{align*}

To proceed, we show that 
$\sum_{i = 1}^{d_s}\E_{\pi_i} s_h s_h^\top \succeq \underline{\lambda} I$. 

From Assumption \ref{aspt:LQR_cover}, we have 
$
    \mathbb{E}_{\pi_i^*}[s_h^\top e_i e_i^\top s_h - a_ha_h^\top] \succeq b_3^2b_5^2 - b_5^2,
$
and by the fact that $\pi_i$ is a $(b_3^2-1)b_5^2/2$-optimal policy, we have
$$
    \mathbb{E}_{\pi_i^*}[s_h^\top e_i e_i^\top s_h - a_ha_h^\top] \succeq (b_3^2b_5^2 - b_5^2)/2.
$$

Since $\E_{\pi_i} a_ha_h^\top \succeq 0$, we have
$
    \mathbb{E}_{\pi_i}[s_h^\top e_i e_i^\top s_h] \succeq (b_3^2-1)b_5^2/2.
$
Therefore, $\sum_{i = 1}^{d_s}\E_{\pi_i} s_h s_h^\top \succeq \underline{\lambda} I$ with $\underline{\lambda} = (b_3^2-1)b_5^2/2$.

Combined with (\ref{equ:cov_decomp}), we have
$$
    \E_{\tilde{\pi}}s_{h+1}s_{h+1}^\top \succeq \frac{\prod_{h'=1}^{h-1}(1-\epsilon_{h'})\epsilon_{h}}{d_s} \left( \underline{\lambda} A_hA_h^\top + d_s\sigma^2 B_hB_h^\top \right).
$$

Apply Assumption \ref{aspt:LQR_cover} again, for each $\nu_i = e_i, i = 1, \dots, d_s$, there exists some policy $\pi_i'$ with $\|a_h\|_2 \leq b_5$, such that 
$
    \nu_i^\top \E_{\pi_i'}s_{h+1}s_{h+1}^\top \nu_i \geq b_3^2b_5^2 - b_5^2.
$ Therefore, we have that $ \sum_{i = 1}^{d_s}\E_{\pi_i'}s_{h+1}s_{h+1}^\top  \succeq (b_3^2-1)b_5^2 I$

The proof is completed by 
\begin{align*}
    \sum_{i = 1}^{d_s}\E_{\pi_i'}s_{h+1}s_{h+1}^\top 
    &\preceq 2\sum_{i = 1}^{d_s} \left(A_h\E_{\pi_i'}s_{h}s_{h}^\top A_h^\top + B_h\E_{\pi_i'}a_{h}a_{h}^\top B_h^\top \right) \\
    &\preceq 2\sum_{i = 1}^{d_s} \left(b_4^2 A_h A_h^\top + b_5^2 B_h\E_{\pi_i'} B_h^\top \right) \\
    &\preceq \frac{2\max\{b_4^2, b_5^2\}}{\max\{\underline{\lambda}, d\sigma^2\}} \frac{\prod_{h'=1}^{h-1}(1-\epsilon_{h'})\epsilon_{h}}{d_s} \mathbb{E}_{\tilde{\pi}} s_{h+1}s_{h+1}^\top.
\end{align*}

To complete the proof of Theorem \ref{thm:LQR}, we combine Lemma \ref{lem:LQR_full-rank} and Lemma \ref{lem:LQR_induction}.

\end{proof}

\subsection{Supporting lemmas}
Lemma \ref{lem:LQR_full-rank} shows that having a full rank covariance matrix for the state $s_h$ is a sufficient condition for bounded occupancy measure.
\begin{lem}
\label{lem:LQR_full-rank}
Let $\mathcal{F}$ be the function class described above. For any policy $\pi$ and $h$ such that 
$$
    \lambda_{\text{min}} (\mathbb{E}_{\pi}[s_h s_h^\top]) \geq \underline{\lambda},
$$
we have for any $\pi'$ such that $\max_{h}\|s_h\|_2 \leq b_4$, and for any $f' \in \mathcal{F}$, 
$$
    \mathbb{E}_{\pi^{\prime}}^M\left[\left(\mathcal{E}_h^2 f^{\prime}\right)\left(s_h, a_h\right)\right] \leq \frac{b_4^2}{\underline{\lambda}^2} \mathbb{E}_\pi^M\left[\left(\mathcal{E}_h^2 f^{\prime}\right)\left(s_h, a_h\right)\right].
$$
\end{lem}

\begin{proof}
Lemma \ref{lem:LQR_quadratic_value} shows that the Bellman error also takes a quadratic form of $s_h$.

\begin{lem}
\label{lem:LQR_quadratic_value}
For any $f\in\mathcal{F}$, there exists some matrix $\tilde{G}_{h}$ such that $(\mathcal{E}_{h} f)(x) = x^\top \tilde{G}_h x$.
\end{lem}

To complete the proof of Lemma \ref{lem:LQR_full-rank}, let $w_h = s_h \otimes s_h$ be the Kronecker product between $s_h$ and itself. By Lemma \ref{lem:LQR_quadratic_value}, we can write 
$
    (\mathcal{E}_{h} f)(s_h) = \operatorname{Vec}(\tilde{G}_{h})^{\top} w_h.
$
Again, this is an analog of the linear form we had for thee linear MDP case. Thus, we can write
$
    (\mathcal{E}^2_{h} f)(s_h) = \operatorname{Vec}(\tilde{G}_{h})^{\top} w_h w_h^{\top} \operatorname{Vec}(\tilde{G}_{h}).
$

By Lemma \ref{lem:kronecker_prod} and the fact that
$
    \mathbb{E}_{\pi} (w_h w_h^\top) = \mathbb{E}_{\pi} (s_h s_h^\top) \otimes \mathbb{E}_{\pi} (s_h s_h^\top)
$, we have $\lambda_{\text{min}}(\mathbb{E}_{\pi} w_h w_h^\top) \geq \underline{\lambda}^2$. 

% \begin{lem}
% \label{lem:LQR_Boundedness}
% Using Assumption
% \end{lem}

For any other policy $\pi'$, and using the fact that $\|w_h\| \leq b_4^2$, we have
$$
    \mathbb{E}_{\pi'}(\mathcal{E}^2_{h} f)(s_h) = \mathbb{E}_{\pi'}[ \operatorname{Vec}(\tilde{G}_{h})^{\top} w_h w_h^{\top} \operatorname{Vec}(\tilde{G}_{h})]  \leq \frac{b_4^2}{\underline{\lambda}^2} \mathbb{E}_{\pi}[\operatorname{Vec}(\tilde{G}_{h})^{\top} w_hw_h^\top \operatorname{Vec}(\tilde{G}_{h})] \leq \frac{b_4^2}{\underline{\lambda}^2} \mathbb{E}_{\pi} (\mathcal{E}^2_h f)(s_h).
$$

\end{proof}

{\bf Lemma \ref{lem:LQR_quadratic_value}.}{ \it
For any $f\in\mathcal{F}$, there exists some matrix $\tilde{G}_{h}$ such that $(\mathcal{E}_{h} f)(x) = x^\top \tilde{G}_h x$.}

\begin{proof} The Bellman error of the LQR can be written as
\begin{align*}
    (\mathcal{E}_{h} f)(x) = \left(x^\top {G}_h x - s^\top R_{h}^s s - a^\top R_{h}^a a - \max_{a' \in \mathbb{R}^{d_a}}  [(A_h s + B_h a)^\top, a^{\prime\top}] {G}_{h+1} \left[\begin{array}{cc}
         A_h s + B_h a  \\
         a' 
    \end{array}\right] \right)
\end{align*}
Note that the optimal $a'$ can be written as some linear transformation of $x$. Thus we can write 
$$
    \max_{a' \in \mathbb{R}^{d_a}}  [(A_h s + B_h a)^\top, a^{\prime\top}] {G}_{h+1} \left[\begin{array}{cc}
         A_h s + B_h a \\
         a' 
    \end{array}\right] = x^\top G' x.
$$
The whole equation can be written as a quadratic form as well.
\end{proof}

\begin{lem}
\label{lem:kronecker_prod}
Let $A \in \mathbb{R}^{d_1\times d_1}$ have eigenvalues $\{\lambda_i\}_{i\in[d]}$ and $B \in \mathbb{R}^{d_2\times d_2}$ have eigenvalues $\{\mu_i\}_{i\in[d]}$. The eigenvalues of $A \otimes B$ are $\{\lambda_i \mu_j\}_{i\in [d_1], j \in d_2}$.
\end{lem}

% \begin{lem}
% Let $\pi$ be any stationary policy. Let $\tilde{\pi}$ be the exploration policy of $\pi$ with Gaussian exploration noise $\eta_h \sim \mathcal{N}(0, \sigma^2 I)$. By the regularity Assumption \ref{aspt:LQR_regularity}, we have 
% $
%     \mathbb{E}_{\tilde{\pi}} [s_h s_h^T] \succ \mathbb{E}_{\pi} [s_h s_h^T].
% $
% \end{lem}

% \begin{proof}
% Let $(\tilde{s}_h)_{h \in [H]}$ be the random trajectory sampled from $\tilde{\pi}$ and $({s}_h)_{h \in [H]}$ be that from $\pi$. Let $(\eta_h)_{h}$ be the exploration noise.
% We can write 
% $$
%     \tilde{s}_h = s_h + \left( \sum_{h' = 1}^{h-1} \Pi_{\tilde{h} = h'+1}^{h-1} A_{\tilde{h}}B_{h'} \eta_{h'} \right) \coloneqq s_h + e_h.
% $$
% We have that $\mathbb{E}[e_h] = 0$. Then
% $
%     \mathbb{E} \tilde{s}_h\tilde{s}_h^T = s_h s_h^T + \mathbb{E}e_he_h^T \succ s_h s_h^T.
% $
% \end{proof}

\section{Relaxing Visibility Assumption}
\label{app:remove_aspt2}
\subsection{Tabular Case}
A simple but interesting case to study is the tabular case, where the value function class is the class of any bounded functions, i.e. $\mathcal{F}_h = \{f: \mathcal{S} \times \mathcal{A} \mapsto [0, 1]\}$. A commonly studied family of multitask RL is the MDPs that share the same transition probability, while they have different reward functions, this problem is studied in a related literature called reward-free exploration \citep{jin2020reward,wang2020reward,chen2022statistical}. Specifically, \citep{jin2020reward} propose to learn $S \times H$ sparse reward MDPs separately and generates an offline dataset, with which one can learn a near-optimal policy for any potential reward function. With a similar flavor, we show that any superset of the $S \times H$ sparse reward MDPs has low myopic exploration gap. Though the tabular case is a special case of the linear MDP case, the lower bound we derive for the tabular case is slightly different, which we show in the following section.

We first give a formal definition on the sparse reward MDP.
\begin{defn}[Sparse Reward MDPs]
\label{def:tabular_sparse}
Let $\mathcal{M}$ be a set of MDPs sharing the same transition probabilities. We say $\mathcal{M}$ contains all the sparse reward MDPs if for each $s, h \in \mathcal{S} \times [H]$, there exists some MDP $M_{s, h} \in \mathcal{M}$, such that $R_{h', M_{s, h}}(s', a') = \mathbbm{1}(s = s', h = h')$ for all $s', a', h'$.
\end{defn}

To show a lower bound on the myopic exploration gap, we  make a further assumption on the occupancy measure $\mu_{h}^{\pi}(s, a) \coloneqq \operatorname{Pr}_{\pi}(s_h = s, a_h = a)$, the probability of visiting $s, a$ at the step $h$ by running policy $\pi$.

\begin{aspt}[Lower bound on the largest achievable occupancy measure]
For all $s, h \in \mathcal{S} \times [H]$, we assume that $\max_{\pi}\mu^{\pi}_h(s) \geq b_1$ for some constant $b$ or $\max_{\pi}\mu^{\pi}_h(s) = 0$.
\label{aspt:ocp_lb_tabular}
\end{aspt}

Assumption~\ref{aspt:ocp_lb_tabular} guarantees that any $\beta$-optimal policy (with $\beta < b_1$) is not a vacuous policy and it provides a lower bound on the corresponding occupancy measure. We will discuss later in Appendix~\ref{app:remove_aspt2} on how to remove this assumption with an extra $S \times H$ factor on the sample complexity bound.

\begin{prop}
\label{prop:sparse_reward}
Consider a set of sparse reward MDP as in Definition \ref{def:tabular_sparse}. Assume Assumption \ref{aspt:ocp_lb_tabular} is true. For any $\beta \leq b_1 / 2$ and $f \in \mathcal{F}_{\beta}$, we have $\alpha(f, \mathcal{F}, \mathcal{M}) \geq \bar \alpha$ for some constant $\bar \alpha = \sqrt{{\beta^2}/({2 e|\mathcal{M}|AH})}$ by choosing $\epsilon_h = 1/h$.
\end{prop}

\begin{proof}
We prove this lemma in a layered manner. Let $h'$ be the minimum step such that there exists some $M_{s, h'}$ is $\beta$-suboptimal. By definition, in the layer $h' - 1$, all the MDPs are $\beta$-suboptimal, in which case $\bm{\pi}_{M_{s, h'-1}}$ visits $(s, h'-1)$ with a probability at least $b / 2$. Now we show that the optimal policy $\pi^*_{M_{s, h'}}$ of a suboptimal MDP $M_{s, h'}$ has lower bounded occupancy ratio. 

For a more concise notation, we let $M' = M_{s, h'}$. Note that 
\begin{align*}
    \mu_{h'}^{\pi_{M'}^*}(s) 
    &= \sum_{s' \in \mathcal{S}} \mu_{h'-1}^{\pi_{M'}^*}(s') P_{h'-1}(s \mid s', \pi_{M'}^*(s'))\\
    &\leq \sum_{s' \in \mathcal{S}} \max_{\pi \in \Pi}\mu_{h'-1}^{\pi}(s') P_{h'-1}(s \mid s', \pi_{M'}^*(s')) \\
    &(\text{By the fact that $\mu_{h'-1}^{\bm{\pi}_{M_{s', h'-1}}}(s')$ is $\beta$-optimal policy of $M_{s', h'-1}$})\\
    &\leq \sum_{s' \in \mathcal{S}} \frac{b_1}{b_1-\beta}\mu_{h'-1}^{\bm{\pi}_{M_{s', h'-1}}}(s') P_{h'-1}(s \mid s', \pi_{M'}^*(s')) \\
    &\leq \sum_{s' \in \mathcal{S}} \frac{b_1 |\mathcal{M}|A}{(b_1-\beta)(1-\epsilon)^{h'-1} \epsilon}  \mu_{h'-1}^{\expl(\bm{\pi})}(s') P_{h'-1}(s \mid s', \expl(\bm{\pi})(s')) \\
    &= \frac{b_1 |\mathcal{M}|A}{(b_1-\beta)(1-\epsilon)^{h'-1} \epsilon} \mu_{h'}^{\expl(\bm{\pi})}(s) 
\end{align*}
The occupancy measure ratio can be upper bounded by $c = \frac{b_1 |\mathcal{M}|A}{(b_1-\beta)(1-\epsilon)^{h'-1} \epsilon}$. Then the myopic exploration gap can be lower bounded by
$$
    \frac{\beta}{\sqrt{c}} = \sqrt{\frac{(b_1-\beta)\beta^2(1-\epsilon)^{h'-1} \epsilon}{b_1 |\mathcal{M}|A}} \geq \sqrt{\frac{\beta^2(1-\epsilon)^{h'-1} \epsilon}{2 |\mathcal{M}|A}}.
$$

To proceed, we choose $\epsilon_h = 1/h$, which leads to 
$
    (1-\epsilon_h)^{h-1} \epsilon \geq 1/(e H).
$
\end{proof}

Plugging this into Theorem \ref{thm:generic_result}, we achieve a sample complexity bound of 
$\mathcal{O}({S^2AH^5}/{\beta^2})$, with $|\mathcal{M}| = SH$. This is not a near-optimal bound for reward-free exploration (a fair comparison in our setup). This is because the sample complexity bound in Theorem \ref{thm:generic_result} is not tailored for tabular case.

\iffalse
\paragraph{Relaxing coverage assumptions.} Though Assumption \ref{aspt:ocp_lb_tabular} is handy for our proof as it guarantees that any $b_1/2$-optimal policy has a probability at least $b_1/2$ to visit the their goal state, this assumption does not appear to be natural. Intuitively, without this assumption, a $b_1$-suboptimal policy can be an arbitrary policy and we can have at most $S$ such policies in total, which may lead to a cumulative error of $S$. In Appendix \ref{app:remove_aspt2}.
\fi

\subsection{Removing coverage assumption}

Though Assumption \ref{aspt:linear_cover} and Assumption \ref{aspt:LQR_cover} are relatively common in the literature, we have not seen an any like Assumption \ref{aspt:ocp_lb_tabular}. In fact, Assumption \ref{aspt:ocp_lb_tabular} is not a necessary condition for sample-efficient myopic exploration as we will discuss in this section. The main technical invention is to construct a mirror transition probability that satisfies the conditions in Assumption \ref{aspt:ocp_lb_tabular}. However, we will see that a inevitable price of an extra $SH$ factor has to be paid. 

To illustrate the obstacle of removing Assumption \ref{aspt:ocp_lb_tabular}, recall that the proof of Proposition \ref{prop:sparse_reward} relies on the fact that all $\beta$-optimal policies guarantee a non-zero probability of visiting the state corresponding to their sparse reward with $\beta < b_1 / 2$. Without Assumption \ref{aspt:ocp_lb_tabular}, a $\beta$-optimal policy can be an arbitrary policy. At the step $h$, we have at most $S$ such MDPs, which may accumulate an irreducible error of $S\beta$, which means that at the round $h+1$, we can only guarantee $S\beta$-optimal policies. An naive adaptation will require us to set the accuracy $\beta' = \beta / S^{H}$ in order to guarantee a $\beta$ error in the last step. The following discussion reveals that the error does not accumulate in a multiplicative way.

\paragraph{Mirror MDP construction.} It is helpful to consider a mirror transition probability modified from our original transition probability. We denote the original transition probability by $P = \{P_{h}\}_{h \in [H]}$. Consider a new MDP with transition $P' = \{P'_h\}_{h \in [H]}$ and state space $\mathcal{S}' = \mathcal{S} \cup \{s_0\}$, where $s_0$ is a dummy state. We initialize $P'$ such that
\begin{equation}
\label{equ:mirror_initial}
P_{h}'(s' \mid s, a) = P_{h}(s' \mid s, a) \text{ for all $s', s, a, h$, where $s', s \neq s_0$, and }\\
P'_h(s_0 \mid s_0, \cdot) = 1
\end{equation}

Starting from $h=1$, we update $P'_h$ by a forward induction according to Algorithm \ref{alg:mirro_MDP}. The design principle is to direct the probability mass of visiting $(s, h+1)$ to $(s_0, h+1)$, whenever the maximal probability of visiting $(s, h+1)$ is less than $\beta$.

\begin{algorithm}[h]
\caption{Creating Mirror Transitions}
\begin{algorithmic}
\State \textbf{Input:} Original Transition $P$, threshold $\beta > 0$.
\State Initialize $P'$ according to (\ref{equ:mirror_initial})
\For{$h = 1, 2, \dots, H-1$}
\For{each $s \in \mathcal{S}$ such that $\max_{\pi} \mu_{h+1}^{\prime \pi}(s) \leq \beta$}
\State $P'_h(s_0 \mid \tilde{s}, \tilde{a}) \leftarrow P'_h(s_0 \mid \tilde{s}, \tilde{a}) + P'_h(s \mid \tilde{s}, \tilde{a})$ for each $\tilde{s}$, $\tilde{a}$.
\State $P'_h(s \mid \tilde{s}, \tilde{a}) \leftarrow 0$ for each $\tilde{s}$, $\tilde{a}$.
\EndFor
\EndFor
\State \textbf{Return} $P'$
\end{algorithmic}
\label{alg:mirro_MDP}
\end{algorithm}

By definition of $P'$, we have two nice properties.
\begin{prop}
For any $h \in [H]$, $s \in \mathcal{S}$, we have $\max_{\pi} \mu^{\prime \pi}_{h}(s) = 0$ or $\max_{\pi} \mu^{\prime \pi}_{h}(s) > \beta$.
\end{prop}

Thus, $P'$ nicely satisfies our Assumption \ref{aspt:ocp_lb_tabular}. We also have that $P'$ is not significantly different from $P$.

\begin{prop}
For any policy $\pi$, $\mu^{\prime \pi}_{h}(s) \geq \mu^{\pi}_{h}(s) - HS\beta$. Further more, any $(SH+1)\beta$-suboptimal policy for $P$ is at least $\beta$-suboptimal for $P'$ with respect to the same reward.
\end{prop}
\begin{proof}
We simply observe that $\max_{\pi} \mu_h^{\prime \pi}(s_0) \leq (h-1)S\beta$. This is true since at any round, we have at most $S$ states with $\max_{\pi} \mu^{'\pi}(s) \leq \beta$, all the probability that goes to $s$ will be deviated to $s_0$. Therefore, for any $\pi$ 
$$
    \mu_{h+1}^{\prime \pi}(s_0) \leq \mu_{h}^{\prime \pi}(s_0) + S\beta.
$$
\end{proof}

Therefore, any $(SH+1)\beta$-suboptimal policy for $P$ has the myopic exploration gap of $\beta$-suboptimal policy for $P'$.

\begin{thm}
Consider a set of sparse reward MDP as in Definition \ref{def:tabular_sparse}. For any $\beta \in (0, 1)$ and $f \in \mathcal{F}_{\beta}$, we have $\alpha(f, \mathcal{F}, \mathcal{M}) \geq \bar \alpha$ for some constant $\bar \alpha =  \Omega( \sqrt{{\beta^2}/({ |\mathcal{M}|AS^2H^3})})$ by choosing $\epsilon_h = 1/(h+1)$.
\end{thm}

\section{Connections to Diversity}
\label{appendix:bellman-eluder}
Diversity has been an important consideration for the generalization performance of multitask learning. How to construct a diverse set, with which we can learn a model that generalizes to unseen task is studied in the literature of multitask supervised learning.

Previous works \citep{tripuraneni2020theory,xu2021representation}  have studied the importance of diversity in multitask representation learning. They assume that the response variable is generated through mean function $f_t \circ h$, where $h$ is the representation function shared by different tasks and $f_t$ is the task-specific prediction function of a task indexed by $t$. They showed that diverse tasks can learn the shared representation that generalizes to unseen downstream tasks. More specifically, if $f_t \in \mathcal{F}$ is a discrete set, a diverse set needs to include all possible elements in $\mathcal{F}$. If $\mathcal{F}$ is the set of all bounded linear functions, we need $d$ tasks to achieve a diverse set. Note that 
these results align with the results presented in the previous section. \textit{Could there be any connection between the diversity in multitask representation learning and the efficient myopic exploration?}

\cite{xu2021representation} showed that Eluder dimension is a measure for the hardness of being diverse. Here we introduce a generalized version called distributional Eluder dimension \citep{jin2021bellman}. 

\begin{defn}[$\varepsilon$-independence between distributions] 
Let $\mathcal{G}$ be a class of functions defined on a space $\mathcal{X}$, and $\nu, \mu_1, \ldots, \mu_n$ be probability measures over $\mathcal{X}$. We say $\nu$ is $\varepsilon$-independent of $\left\{\mu_1, \mu_2, \ldots, \mu_n\right\}$ with respect to $\mathcal{G}$ if there exists $g \in \mathcal{G}$ such that $\sqrt{\sum_{i=1}^n\left(\mathbb{E}_{\mu_i}[g]\right)^2} \leq \varepsilon$, but $\left|\mathbb{E}_\nu[g]\right|>\varepsilon$
\end{defn}

\begin{defn}[Distributional Eluder (DE) dimension] Let $\mathcal{G}$ be a function class defined on $\mathcal{X}$, and $\Pi$ be a family of probability measures over $\mathcal{X}$. The distributional Eluder dimension $\operatorname{dim}_{\mathrm{DE}}(\mathcal{G}, \Pi, \varepsilon)$ is the length of the longest sequence $\left\{\rho_1, \ldots, \rho_n\right\} \subset \Pi$ such that there exists $\varepsilon^{\prime} \geq \varepsilon$ where $\rho_i$ is $\varepsilon^{\prime}$-independent of $\left\{\rho_1, \ldots, \rho_{i-1}\right\}$ for all $i \in[n]$.
\end{defn}

\begin{defn}[Bellman Eluder (BE) dimension (Jin et al., 2021)] Let $\mathcal{E}_h \mathcal{F}$ be the set of Bellman residuals induced by $\mathcal{F}$ at step $h$, and $\Pi=\left\{\Pi_h\right\}_{h=1}^H$ be a collection of $H$ probability measure families over $\mathcal{X} \times \mathcal{A}$. The $\varepsilon$-Bellman Eluder dimension of $\mathcal{F}$ with respect to $\Pi$ is defined as
$$
\operatorname{dim}_{\mathrm{BE}}(\mathcal{F}, \Pi, \varepsilon):=\max _{h \in[H]} \operatorname{dim}_{\mathrm{DE}}\left(\mathcal{E}_h \mathcal{F}, \Pi, \varepsilon\right)
$$
\end{defn}

\paragraph{Constructing a diverse set.} For each $h \in [H]$, consider a sequence of functions $f_1, \dots, f_d \in \mathcal{F}$, such that the induced policy $(\pi^{f_i})_{i \in [d]}$ generates probability measures $(\mu_{h+1}^{f_i})_{i \in [d]}$ at the step $h+1$. Let $(\mu_{h+1}^{f_i})_{i \in [d]}$ be $\epsilon$-independence w.r.t the function class $\mathcal{E}_{h}\mathcal{F}$ between their predecessors. By the definition of BE dimension, we can only find at most $\operatorname{dim}_{\mathrm{DE}}\left(\mathcal{E}_h \mathcal{F}, \Pi, \varepsilon\right)$ of these functions. Then conditions in Definition \ref{defn:myopic_exploration_gap} is satisfied with $c = 1/(dH)$.

\paragraph{Revisiting linear MDPs.} The task set construction in \ref{def:linear_diverse} seems to be quite restricted as we require a set of standard basis. One might conjecture that a task set $M_{i, h}$ with full rank $[\theta_{1, h}, \dots, \theta_{d, h}]$ will suffice. From what we discussed in the general case, we will need the occupancy measure generated by the optimal policies for these MDPs to be $\epsilon$-independent and any other distribution is $\epsilon$-dependent. This is generally not true even if the reward parameters are full rank. To see this, let us consider a tabular MDP case with two states $\{1, 2\}$, where at the step $h$, we have two tasks $M_1$, $M_2$, with $R_{h, M_1}(s, a) = \mathbbm{1}[s = 1]$ and $R_{h, M_2}(s, a) = 0.51\mathbbm{1}[s = 1] + 0.49\mathbbm{1}[s = 2]$. This gives $\theta_{h, M_1} = [1, 0]$ and $\theta_{h, M_2} = [0.49, 0.51]$ as shown in Figure \ref{fig:diverse_illustration}.

\begin{figure}[h]
    \centering
    \includegraphics[width = 0.4\textwidth]{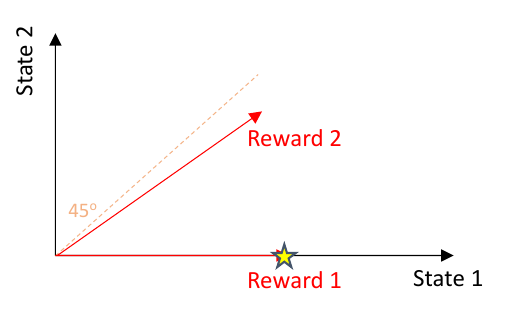}
    \caption{An illustration of why a full-rank set of reward parameters does not achieve diversity. The red arrows are two reward parameters and the star marks the generated state distributions of the optimal policies corresponding to the two rewards at the step $h$. Since both optimal policies only visit state 1, they may not provide a sufficient exploration for the next time step $h+1$.}
    \label{fig:diverse_illustration}
\end{figure}

Construct the MDP such that the transition probability and action space any policy either visit state 1 or state 2 at the step $h$. Then the optimal policies for both tasks are the same policy which visits state $1$ with probability one, even if the reward parameters $[\theta_{h, M_1}, \theta_{h, M_2}]$ are full-rank.

\bibliography{main}
\bibliographystyle{apalike}
\end{document}